\documentclass[11pt]{article}
\usepackage[english]{babel}
\usepackage[T1]{fontenc}
\usepackage[utf8]{inputenc}
\usepackage{lmodern}
\usepackage[a4paper,left=1.2in,right=1.2in,top=1.1in,bottom=1.1in,footskip=0.5in]{geometry}
\usepackage{amsthm}
\usepackage{amsmath,amsfonts,amssymb}
\usepackage{graphicx}
\usepackage{caption}
\usepackage{subcaption}
\usepackage{authblk} 
\usepackage{url}            
\usepackage{booktabs}       
\usepackage{nicefrac}       
\usepackage{microtype}      

\usepackage{enumerate,xspace}
\usepackage{hyperref} 
\hypersetup{linkcolor=blue, citecolor=blue, urlcolor=blue, colorlinks=true}

\usepackage{psfig,epsf}

\RequirePackage{comment}

\theoremstyle{plain}
\newtheorem{theorem}{Theorem}[section]
\newtheorem{lemma}[theorem]{Lemma}
\newtheorem{proposition}[theorem]{Proposition}

\newtheorem{corollary}[theorem]{Corollary}

\newtheorem*{theorem*}{Theorem}
\newtheorem*{lemma*}{Lemma}
\newtheorem*{proposition*}{Proposition}
\newtheorem*{conjecture*}{Conjecture}
\newtheorem{fact*}{Fact}

\theoremstyle{definition}

\newtheorem{assumption}[theorem]{Assumption}
\newtheorem*{definition*}{Definition}
\newtheorem*{question*}{Question}
\newtheorem*{example*}{Example}
\newtheorem*{remark*}{Remark}
\newtheorem*{remarks*}{Remarks}
\newtheorem*{exercise*}{Exercise}
\newtheorem*{assumption*}{Assumption}

\usepackage{times}

\usepackage[section]{algorithm}
\usepackage{algorithmic}
\ifpdf
  \DeclareGraphicsExtensions{.eps,.pdf,.png,.jpg}
\else
  \DeclareGraphicsExtensions{.eps}
\fi

\usepackage{amsmath,amssymb}

\DeclareMathOperator{\diag}{diag}

\newcommand{\weq}{\ = \ }
\newcommand{\cond}{\, | \,}
\newcommand{\wle}{\ \le \ }
\newcommand{\wge}{\ \ge \ }

\newcommand{\gesim}{\gtrsim}

\newcommand{\cC}{\mathcal{C}}
\newcommand{\cI}{\mathcal{I}}
\newcommand{\cL}{\mathcal{L}}

\newcommand{\cN}{\mathcal{N}}

\newcommand{\cP}{\mathcal{P}}
\newcommand{\R}{\mathbf{R}}
\DeclareMathOperator*{\argmax}{arg\,max}
\DeclareMathOperator*{\argmin}{arg\,min}

\newcommand{\pin}{p_{\rm in}}
\newcommand{\pout}{p_{\rm out}}

\newcommand{\cin}{c_{\rm in}}
\newcommand{\cout}{c_{\rm out}}

\newcommand{\E}{\mathbb{E}}
\newcommand{\pr}{\mathbb{P}}

\newcommand{\hX}{\widehat{X} }

\newcommand{\hZ}{\widehat{Z} }

\newcommand{\bx}{\bar{x} }
\newcommand{\bgamma}{\bar{\gamma} }
\newcommand{\balpha}{\bar{\alpha} }
\newcommand{\bd}{\bar{d} }
\newcommand{\bdelta}{\bar{\delta} }
\newcommand{\bb}{\bar{b} }
\newcommand{\bs}{\bar{s} }

\newcommand{\Ber}{\operatorname{Ber}}

\newcommand{\Unif}{\operatorname{Uni}}
\newcommand{\MAP}{\operatorname{MAP}}

\usepackage{comment}
\usepackage{xcolor}
\usepackage{subcaption}
\usepackage{empheq,etoolbox}

\usepackage{graphicx}
\usepackage{import}

\newcommand{\tv}{\tilde{v}}

\ifpdf
\hypersetup{
  pdftitle={Almost exact recovery in noisy semi-supervised learning},
  pdfauthor={K. Avrachenkov and M. Dreveton}
}
\fi

\begin{document}

\title{Almost exact recovery in noisy semi-supervised learning}

\author{Konstantin Avrachenkov\footnote{ \texttt{k.avrachenkov@inria.fr}} } 
\author{Maximilien Dreveton\footnote{ \texttt{maximilien.dreveton@epfl.ch} \\ The paper is accepted in \textit{Probability in the Engineering and Informational Sciences}.} }

\affil{Inria Sophia Antipolis, France}
\renewcommand\Authands{ and }

\date{}	
\numberwithin{equation}{section}
	
\maketitle

\begin{abstract}
Graph-based semi-supervised learning methods combine the graph structure and labeled data to classify unlabeled data.
In this work, we study the effect of a noisy oracle on classification. In particular, we derive the Maximum A Posteriori (MAP) estimator for clustering a Degree Corrected Stochastic Block Model (DC-SBM) when a noisy oracle reveals a fraction of the labels. We then propose an algorithm derived from a continuous relaxation of the MAP, and we establish its consistency.
Numerical experiments show that our approach achieves promising performance on synthetic and real data sets, even in the case of very noisy labeled data.

\bigskip


\bigskip

\noindent

\end{abstract}

\section{Introduction}
Semi-supervised learning (SSL) aims at achieving superior learning performance by combining unlabeled and labeled data.
Since typically the amount of unlabeled data is large compared to the amount of labeled data, SSL methods are relevant when the performance of unsupervised learning is low, or when the cost of getting a large amount of labeled data for supervised learning is too high.
Unfortunately, many standard SSL methods have been shown to not efficiently use the unlabeled data, leading to unsatisfactory or unstable performance \cite[Chapter~4]{chapelle2006semi}, \cite{Ben-David_Lu_Pal_2008,Cozman_Cohen_2002}. 
Moreover, noise in the labeled data can further degrade the performance. In practice, the noise can come from a tired or non-diligent expert carrying out the labeling task or even from adversarial data corruption.

In this paper, we investigate the problem of graph clustering, where one aims to group the nodes of a graph into different classes. Our working model is the two-class Degree Corrected Stochastic Block Model (DC-SBM), with side information on some node’s community assignment given by a noisy oracle. The DC-SBM was introduced in~\cite{karrer2011stochastic} to account for degree heterogeneity and block structure. Let $n$ be the number of nodes. Each node $i \in [n]$ is given a community label $Z_i \in \{-1,1\}$ chosen uniformly at random and a parameter $\theta_i >0$. Given $Z = (Z_1, \dots, Z_n)$ and $\theta = \left( \theta_1, \dots, \theta_n \right)$, an undirected edge is added between nodes $i$ and $j$ with probability $\min( 1, \theta_i \theta_j \pin)$, if $Z_i = Z_j$, and with probability $\min(1, \theta_i\theta_j \pout)$, otherwise. This model reduces to the standard Stochastic Block Model (SBM) \cite{Abbe_2017} if $\theta_i = 1$ for every node~$i$.
The unsupervised clustering problem consists of inferring the latent community structure~$Z$ given one observation of a DC-SBM graph. We make the problem semi-supervised by introducing a noisy oracle. For every node, this oracle reveals the correct community label with probability~$\eta_1$, a wrong community label with probability~$\eta_0$, and reveals nothing with probability $1-\eta_1-\eta_0$.

We first derive the Maximum A Posteriori (MAP) estimator for SSL-clustering in a DC-SBM graph given the \textit{a priori} information induced by a noisy oracle and graph structure. We note that, despite its simplicity, this result did not appear previously in the literature, neither for a perfect oracle nor for SBM. In particular, we show that the MAP is the solution to a minimization problem that involves a trade-off between three factors: a cut-based term (as in the unsupervised scenario), a regularization term (penalizing solutions with unbalanced clusters), and a loss term (penalizing predictions that differ from the oracle information).

As solving the MAP estimator is NP-hard, we propose a continuous relaxation and derive an SSL version of a spectral method based on the adjacency matrix. We establish a bound on the ratio of misclassified nodes for this continuous relaxation, and we show that this ratio goes to zero under the hypothesis that the average degree diverges and an almost perfect oracle (see Corollary~\ref{cor:almost_exact_recovery} for a rigorous statement). As a result, the proposed SSL method guarantees almost exact recovery (recovering all but $o(n)$ labels when $n$ goes to infinity) even when a part of the side information is incorrect. We note that even though we work with the case of two clusters, most of our results are extendable to the setting of more than two clusters at the expense of more cumbersome notations.

One can make several parallels between our continuous relaxation and state-of-the-art techniques. Indeed, SSL-clustering often relies on minimization frameworks (see \cite{avrachenkov2012generalized,chapelle2006semi} for an overview). The idea of minimizing a well-chosen energy function was proposed in~\cite{zhu2003semi}, under the constraint of keeping the labeled nodes’ predictions equal to the oracle labels. 
As we show in the numerical section, this hard constraint is unsuitable if the oracle reveals false information. Consequently, \cite{belkin2004regularization} introduced an additional loss term in the energy function to allow the prediction to differ from the oracle information. We recover this loss term with an additional theoretical justification because it comes from a relaxation of the MAP.

Moreover, the regularization term is necessary to prevent the solution from being flat and making classification rely on second-order fluctuations. This phenomenon was previously observed by \cite{Nadler2009SemisupervisedLW} in the limit of an infinite amount of unlabeled data, as well as by \cite{mai2018random} in the large dimension limit. The regularization term here consists of subtracting a constant term from all the entries of the adjacency matrix. It resembles previous regularization techniques, like the centring of the adjacency matrix proposed in~\cite{Mai_Couillet_2020}. However, contrary to \cite{Mai_Couillet_2020}, we study a noisy framework without assuming a large-dimension asymptotic regime. 
Moreover, we solve exactly the relaxed minimization problem instead of giving a heuristic with an extra parameter.

It was shown in \cite{Saad_Nosratinia_2018} that
even with a perfect oracle revealing a constant fraction of the labels, the phase transition phenomena for exact recovery in SBM (recovering all the correct labels with high probability) remains unchanged. Thus, for the exact recovery problem, one could discard all the side information and simply use unsupervised algorithms when the number of data points goes to infinity. Of course, wasting potentially valuable information is not entirely satisfactory. Thus, in the present work, we consider the case of almost exact recovery and an oracle with noisy information.
In \cite{banerjee2023pagerank,kadavankandy2017power} criteria different from the exact recovery
have also been considered in the framework of semi-supervised learning.

The paper is structured as follows. We introduce the model and main notations in Section~\ref{section:MAP}, along with the derivation
of the MAP estimator (Section~\ref{subsection:MAP_estimator_DCSBM}). A continuous relaxation of the MAP is presented in Section~\ref{section:continuous_relaxation} as well as the guarantee of its convergence to the true community structure (Subsection~\ref{subsection:ratio_misclassified_nodes}). We postpone some proofs to the Appendix and leave in the main text only those we consider important to the material exposition. We conclude the paper with numerical results (Section~\ref{section:numerical_results}), emphasizing the effect of the noise on the clustering accuracy. 
In particular, we outperform state-of-the-art graph-based SSL methods in a difficult regime (few label points or large noise).

Lastly, the present paper is a follow-up work on~\cite{avrachenkov2019almost}. However, there
are very important developments. In~\cite{avrachenkov2019almost} we have only established almost exact recovery on SBM for Label Spreading~\cite{zhou2004learning} heuristic algorithm with a linear number of labeled nodes (see~\cite[Assumption 3]{avrachenkov2019almost}). In the present work, we extend the analysis to DC-SBM, investigate the effect of noisy labeled data, and allow a potentially sublinear number of labeled nodes. We also add experiments with real and synthetic data that illustrate our theoretical results.

\section{MAP estimator in a noisy semi-supervised setting}\label{section:MAP}

\subsection{Problem formulation and notations}

A homogeneous \textit{degree corrected stochastic block model} (DC-SBM) is parametrized by the number of nodes $n$, two class-affinity parameters $\pin, \pout$, and a pair $(\theta, Z)$ where $\theta \in \R^n$ is a vector of intrinsic connection intensities and $Z \in \{-1,1\}^n$ is the community labeling vector.
Given $(\pin, \pout, \theta, Z)$, the graph adjacency matrix $A = (a_{ij})$ is generated as
\begin{align}
\label{def:DCSBM}
A_{ij} = A_{ji} \sim 
\left\{
\begin{array}{ll}
\Ber \left( \theta_i \theta_j \pin \right), & \qquad \mathrm{if}\quad Z_i = Z_j, \\
\Ber \left( \theta_i \theta_j \pout \right), & \qquad \mathrm{otherwise,} \\
\end{array}
\right.
\end{align}
for $i \not= j$, and $A_{ii} = 0$.
We assume throughout the paper that $Z_i \sim \Unif \left( \{-1,1\} \right) $, and that the entries of $\theta$ are independent random variables satisfying $\theta_i \in [ \theta_{\min}, \theta_{\max} ]$ with $\E \theta_i = 1$, $\theta_{\min} > 0$, and $\theta_{\max}^2 \max(\pin, \pout) \leq 1$.
In particular, when all the $\theta_i$'s are equal to one, the model reduces to the Stochastic Block Model (SBM):
\begin{align}
A_{ij} = A_{ji} \sim 
\left\{
\begin{array}{ll}
\Ber \left( \pin \right), & \qquad \mathrm{if}\quad Z_i = Z_j, \\
\Ber \left( \pout \right), & \qquad \mathrm{otherwise.} \\
\end{array}
\right.
\end{align}

In addition to the observation of the graph adjacency matrix $A$, an oracle gives us extra information about the cluster assignment of some nodes. This can be represented as a vector $s \in \{0,-1,1\}^n$, whose entries $s_{i}$ are independent and distributed as follows:
\begin{align}
\label{eq:def_oracle}
s_{i} \weq  \left\{
\begin{array}{ll}
Z_i, & \qquad \mathrm{with \ probability} \quad \eta_1, \\
-Z_i, & \qquad \mathrm{with \ probability} \quad \eta_0, \\
0, & \qquad \mathrm{otherwise}.
\end{array}
\right.
\end{align}
In other words, the oracle~\eqref{eq:def_oracle} reveals the correct cluster assignment of node~$i$ with probability~$\eta_1$ and gives a false cluster assignment with probability~$\eta_0$. It reveals nothing with probability~$1-\eta_1-\eta_0$.
The quantity $\pr\left( s_i \not = Z_i \cond s_i \not= 0 \right)$ is the rate of mistakes of the oracle (\textit{i.e.,} the probability that the oracle reveals false information given that it reveals something), and is equal to $\eta_0/(\eta_1+\eta_0)$. The oracle is informative if this quantity is less than~$1/2$, which is equivalent to $\eta_1 > \eta_0$. In the following, we will always assume that the oracle is informative.

\begin{assumption}
\label{assumption:informative_oracle}
 The oracle is informative, that is, $\eta_1 > \eta_0$.
\end{assumption}

Given the observation of $A$ and $s$, the goal of clustering is to recover the community labeling vector $Z$. For an estimator $\hZ \in \{-1,1\}^n$ of $Z$, the relative error is defined as the proportion of misclassified nodes
\begin{align}
\label{eq:def_error}
 L \left( \hZ, Z \right) \weq \frac{1}{n}  \sum_{i=1}^n 1\left( \hZ_i \not= Z_i \right).
\end{align}
Note that, unlike unsupervised clustering, we do not take a minimum over the permutations of the predicted labels since we should be able to learn the correct community labels from the informative oracle.

\textit{Notations}
Given an oracle $s$, we let $\ell$ be the set of labeled nodes, that is $\ell := \{ i \in V : s_i \not= 0 \}$, and denote $\cP$ the diagonal matrix with entries $\left( \cP \right)_{ii} = 1$, if $i \in \ell$, and $\left(\cP \right)_{ii} = 0$, otherwise.


The notation $I_n$ stands for the identity matrix of size $n \times n$, and $1_n$ (resp.,~$0_n$) is the vector of size $n \times 1$ of all ones (resp., of all zeros).

For any matrix $A = \left(a_{ij}\right)_{i \in [n], j \in [m]}$ and two sets $S \subset [n]$, $T \subset [m]$, we denote $A_{S,T} = \left(a_{ij}\right)_{i \in S, j \in T}$ the matrix obtained from $A$ by keeping elements, whose row indices are in $S$ and column indices are in $T$. We denote by $\|x\|$ the Euclidean norm of a vector $x$ and by $\|A\|$ the spectral norm of a matrix 
$A \in \R^{ n\times m}$.
Finally, $A \odot B$ refers to the entry-wise matrix product between two matrices $A$ and $B$ of the same size.

\subsection{MAP estimator for semi-supervised recovery in DC-SBM}
\label{subsection:MAP_estimator_DCSBM}

Given a realization of a DC-SBM graph adjacency matrix $A$ and the oracle information $s$, the Maximum A Posteriori (MAP) estimator is defined as
\begin{align}
\label{eq:MAP_definition}
\hZ^{\MAP} & \weq \argmax_{ z \in \{-1,1\}^n } \pr( z \cond A, s).
\end{align}
This estimator is known to be optimal (in the sense that if it fails then any other estimator would also fail, see \textit{e.g.}, \cite{Iba_1999}) for the exact recovery of all the community labels.
Theorem~\ref{thm:MAP_dcsbm} provides an expression of the MAP.

\begin{theorem}\label{thm:MAP_dcsbm} 
 Let $G$ be a graph drawn from DC-SBM as defined in~\eqref{def:DCSBM} and $s$ be the oracle information as defined in~\eqref{eq:def_oracle}.
 Denote 
 $ 
 M \weq  (F_1-F_0) \odot A + F_0,
 $
 where $F_0 = \left(f^{(0)}_{ij} \right)$ and $F_1 = \left(f^{(1)}_{ij} \right)$ such that 
 $f^{(a)}_{ij} = \log \frac{\pr(A_{ij} = a \cond z_i = z_j )}{ \pr(A_{ij} = a \cond z_i \not= z_j ) } $ for $a \in \{0,1\}$.
 The MAP estimator defined in~\eqref{eq:MAP_definition} is given~by
 \begin{equation}
 \label{eq:MAP_dcsbm_SSL_noise}
 \hZ^{\MAP} \weq \argmin_{ \substack { z \in \{-1,1\}^n } } \, 
 \left(  z^T M z + 
 \log \left( \frac{\eta_1}{\eta_0} \right) 
 \left\|  \cP z - s \right\|^2 \right). 
 \end{equation}
 \noindent For a perfect oracle $(\eta_0 = 0)$ this reduces to
 \begin{equation}
 \label{eq:MAP_dcsbm_SSL_no_noise}
 \hZ^{\MAP} \weq \argmin_{\substack{ z \in \{-1,1\}^n \\ z_{\ell} = s_{\ell} } }  z^T M z.
 \end{equation}
\end{theorem}

The proof of Theorem~\ref{thm:MAP_dcsbm} is standard and postponed to Appendix~\ref{appendix_section_derivation_MAP}. We note that, despite being \textit{a priori} standard, this result did not appear previously in the literature (neither for the standard SBM nor for the perfect oracle).

The minimization problem~\eqref{eq:MAP_dcsbm_SSL_noise} consists of a trade-off between minimizing a quadratic function $z^T M z$ and a penalty term. 
This trade-off reads as follows: for each labeled node such that the prediction contradicts the oracle, a penalty $\log \left( \frac{\eta_1}{\eta_0} \right) > 0$ is added.
In particular, when the oracle is uninformative, that is $\eta_1 = \eta_0$, then this term is null, and Expression~\eqref{eq:MAP_dcsbm_SSL_noise} reduces to the MAP for unsupervised clustering.

The following Corollary~\ref{corollary:MAP_SBM}, whose proof is in Appendix~\ref{appendix_section_derivation_MAP}, provides the expression of the MAP estimator for a standard SBM.

\begin{corollary}
\label{corollary:MAP_SBM}
 The MAP estimator for semi-supervised clustering on SBM graph with $\pin > \pout$ and with an oracle $s$ defined in~\eqref{eq:def_oracle} is given by
 \begin{align}
 \label{eq:MAP_SSL_noise}
 \hZ^{\MAP} \weq \argmin_{ z \in \{-1,1\}^n } \, \left( - z^T \left(A - \tau 1_n 1_n^T \right) z + \lambda^*  \left\| \cP z  - s \right\|_2^2 \right),
 \end{align}
 where 
 $
 \tau = \dfrac{\log \Big( \dfrac{1-\pout}{1-\pin} \Big) }{ \log \Big( \dfrac{\pin (1-\pout)}{\pout (1-\pin)} \Big) }$ and $\lambda^* = \dfrac{ \log \Big( \dfrac{\eta_1}{\eta_0} \Big)}{\log \bigg( \dfrac{\pin (1-\pout)}{\pout(1-\pin)} \bigg)}. 
 $
 For the perfect oracle, this reduces to
 \begin{equation}\label{eq:MAP_SSL_no_noise}
 \hZ^{\MAP} \weq \argmin_{\substack{ z \in \{-1,1\}^n \\ z_{\ell} = s_{\ell} } } \, z^T \left( -A + \tau 1_n 1_n^T \right) z.
 \end{equation}
\end{corollary}

\section{Almost exact recovery using a continuous relaxation}\label{section:continuous_relaxation}

As finding the MAP estimate is NP-hard \cite{wagner1993between}, we perform a continuous relaxation (Section~\ref{subsection:continuous_relaxation_MAP}).
We then give an upper bound on the number of misclassified nodes in Section~\ref{subsection:ratio_misclassified_nodes}.

\subsection{Continuous relaxation of the MAP}
\label{subsection:continuous_relaxation_MAP}

For the sake of presentation simplicity, we focus on the MAP for SBM, \textit{i.e.,} minimization problem~\eqref{eq:MAP_SSL_noise}.
We perform a continuous relaxation mirroring what is commonly done for spectral methods~\cite{Newman_2013_spectral}, namely
\begin{align}
\label{eq:MAP_relaxed}
\hX \weq \argmin_{ \substack{ x \in \R^{n} \\ \sum_i \kappa_i x_i^2 = \sum_i \kappa_i } } \left( -x^T A_\tau x + \lambda ( s - \cP x )^T ( s - \cP x ) \right),
\end{align}
where $A_{\tau} = A - \tau 1_n 1_n^T$ and $\kappa = (\kappa_1, \dots, \kappa_n)$ is a vector of positive entries. 
We choose to constrain $x$ on the hyper-sphere $\|x\|^2 = n$ by letting $\kappa_i = 1$, but other choices would lead to a similar analysis. In particular, in the numerical Section~\ref{section:numerical_results} we will compare this choice with a degree-normalization approach (\textit{i.e.,} $\kappa_i = d_i$).

We further note that for the perfect oracle, the corresponding relaxation of~\eqref{eq:MAP_SSL_no_noise} is
\begin{equation}
\label{eq:MAP_relaxed_no_noise}
\hX = \argmin_{\substack{ x \in \R^n \\ x_{\ell} = s_{\ell} \\ \|x\|^2 = n } } \left(  - x^T A_\tau  x \right).
\end{equation}
Given the classification vector $\hX \in \R^{n}$, node $i$ is classified into cluster $\hZ_i \in \{-1, 1\}$ such that
\begin{equation}\label{detection_criteria}
\hZ_i \weq \left\{
	\begin{array}{ll}
	1, & \qquad \mathrm{if}\quad \hX_i > 0, \\
	-1, & \qquad \mathrm{otherwise}. \\
	\end{array}
	\right.
\end{equation}
Let us solve the minimization problem~\eqref{eq:MAP_relaxed}. By letting $\gamma \in \R$ be the Lagrange multiplier associated with the constraint $\|x\|^2 = n$, the Lagrangian of the optimization problem~\eqref{eq:MAP_relaxed}~is
\[
-x^T A_\tau x + \lambda ( s - \cP x )^T ( s - \cP x )    - \gamma \left(  x^T x  - n \right).
\]
This leads to the {\it constrained} linear system
\begin{align}
\label{eq:constraint_linear_system}
 \left\{
 \begin{array}{rc}
 \left( - A_\tau + \lambda \cP - \gamma I_n \right) x \weq & \lambda s, \\
  x^T x \weq & n,
 \end{array}
 \right.
\end{align}
whose unknowns are $\gamma$ and $x$.

While~\cite{Mai_Couillet_2020} let $\gamma$ to be a hyper-parameter (hence the norm constraint $x^T x = n$ is no longer verified), the exact optimal value of $\gamma$ can be found explicitly following~\cite{Gander_Golub_Matt_1989}.
Firstly, we note that if $(\gamma_1,x_1)$ and $(\gamma_2,x_2)$ are solutions of the system~\eqref{eq:constraint_linear_system}, then (see Lemma~\ref{lemma:diff_cost} for the derivations) 
\begin{align*}
    \cC (x_1) - \cC (x_2) \weq \frac{ \gamma_1 - \gamma_2 }{2} \, \left\| x_1 - x_2 \right\|^2,
\end{align*}
where $\cC(x) = - x^T A_\tau x + \lambda (s-\cP x)^T (s-\cP x)$ is the cost function minimized in~\eqref{eq:MAP_relaxed}.
Hence, among the solution pairs $(\gamma, x)$ of the system~\eqref{eq:constraint_linear_system}, the solution of the minimization problem~\eqref{eq:MAP_relaxed} is the vector~$x$ associated with the smallest~$\gamma$.

Secondly, the eigenvalue decomposition of $-A_{\tau} + \lambda \cP$ reads as
\begin{align*}
 -A_{\tau} + \lambda \cP \weq Q \Delta Q^T,
\end{align*}
where $\Delta = \diag ( \delta_1, \dots, \delta_n)$ with $\delta_1 \leq \dots \leq \delta_n$ and $Q^T Q = I_n$. Therefore, after the change of variables $u = Q^T x$ and $b = \lambda Q^T s$, the system~\eqref{eq:constraint_linear_system} is transformed to
\begin{align*}
 \left\{
 \begin{array}{rl}
 \Delta u & \weq \gamma u + b, \\
 u^T u  & \weq n.
 \end{array}
 \right.
\end{align*}
Thus, the solution $\hX$ of the optimization problem~\eqref{eq:MAP_relaxed} satisfies
\begin{align}
 \label{eq:SSL_solution_minimization_noisy}
 \left( - A_\tau + \lambda \cP - \gamma_* I_n \right) \hX & \weq \lambda s,
\end{align}
where $\gamma_*$ is the smallest solution of the \textit{explicit secular equation}~\cite{Gander_Golub_Matt_1989}
\begin{align}
\label{eq:explicit_secular_equation}
    \sum_{i=1}^n \left(  \frac{ b_i }{  \delta_i - \gamma  }  \right)^2 - n \weq 0.
\end{align}

We summarize this in Algorithm~\ref{algo:SSL-SC-regularized_adjacency_matrix}. 
Note that for the sake of generality, we let $\lambda$ and $\tau$ be hyper-parameters of the algorithm.
If the model parameters are known, we can use the expressions of $\lambda$ and $\tau$ derived in Corollary~\ref{corollary:MAP_SBM}. The choice of $\lambda$ and $\tau$ is further discussed in Section~\ref{section:numerical_results}.
We must use power iterations or Krylov subspace methods to apply Algorithm~\ref{algo:SSL-SC-regularized_adjacency_matrix} to large data sets. The main computational bottleneck in those methods will be the matrix-vector product $A_\tau v$. The matrix $A_\tau$ is not sparse. Since $A_\tau$ is a sum of a sparse matrix and a rank-one matrix, the computation of $A_\tau v = Av - \tau (1_n^T v) 1_n$ can be done efficiently by subtracting the same scalar $\tau (1_n^T v)$ from all the entries of the result of the sparse matrix-vector multiplication.

\begin{algorithm}
\caption{Semi-supervised learning with regularized adjacency matrix.}
\label{algo:SSL-SC-regularized_adjacency_matrix}
\begin{algorithmic}
\STATE{\textbf{Input:} Adjacency matrix $A$, oracle information $s$, parameters $\tau$ and $\lambda$.}
\STATE{\textbf{Procedure:}}
\STATE{Let $\gamma^*$ be the smallest solution of Equation~\eqref{eq:explicit_secular_equation}. \\
 Compute $\hX$ as the solution of Equation~\eqref{eq:SSL_solution_minimization_noisy}.}
\FOR{$ i = 1 \dots n$}
\STATE{$\hZ_i = \mathrm{sign}\left( \hX_i \right)$.}
\ENDFOR
\RETURN $\hZ$.
\end{algorithmic}
\end{algorithm}

\subsection{Ratio of misclassified nodes}
\label{subsection:ratio_misclassified_nodes}

 This section gives bounds on the number of unlabeled nodes misclassified by Algorithm~\ref{algo:SSL-SC-regularized_adjacency_matrix}. We then specialize the results for some particular cases.
 
 \begin{theorem}
 \label{thm:bound_number_misclassified_nodes}
 Consider a DC-SBM with a noisy oracle as defined in \eqref{def:DCSBM},\eqref{eq:def_oracle}. Let $\bd = \frac{n}{2}(\pin + \pout)$ and $\balpha = \frac{n}{2}(\pin - \pout)$.
 Suppose that $\tau > \pout$ and that $\eta_0 n \sqrt{\eta_1+\eta_0} \ll \lambda$, and let $\hZ$ be the output of Algorithm~\ref{algo:SSL-SC-regularized_adjacency_matrix}.
 Then, for any $r > 0$, there exists a constant $C$ such that the proportion of misclassified unlabeled nodes satisfies
 \begin{equation*}
 L \left( \hZ_u , Z_u \right)
 \wle C 
 \left( \frac{\pin + \pout }{ \pin - \pout } \right)^2
 \Bigg(\dfrac{ \bar{ \alpha } + \lambda }{  \lambda } \Bigg)^2  \frac{ 1 }{ (\eta_1 + \eta_0) \left( \eta_1 - \eta_0 \right)^2 \bar{d} },
 \end{equation*}
 with probability at least $1 - n^{-r}$.
\end{theorem}

 The value of $\lambda$ in Theorem~\ref{thm:bound_number_misclassified_nodes} serves as a hyper-parameter of the algorithm and may not necessarily be equal to the value $\lambda^*$ computed in Corollary~\ref{corollary:MAP_SBM}. Consequently, one can opt for a $\lambda$ significantly larger than $\eta_0 n \sqrt{\eta_0+\eta_1}$, even if the $\lambda^*$ from Corollary~\ref{corollary:MAP_SBM} is not much larger than $\eta_0 n \sqrt{\eta_0+\eta_1}$. Selecting $\lambda > \lambda^*$ indicates an excessive reliance on the information provided by the oracle, but it has a benign effect on the error bound of the unlabeled nodes given in Theorem~\ref{thm:bound_number_misclassified_nodes}. 

The core of the proof relies on the concentration of the adjacency matrix towards its expectation. This result, as presented in \cite{levina2018concentration}, holds under loose assumptions:
it is valid for any random graph whose edges are independent of each other. To use this result for $\bar{d} = o\big(\log n \big)$, one needs to replace the matrix $A_\tau$ by $A_\tau' = A' - \tau 1_n 1_n^T$, where $A'$ is the adjacency matrix of the graph obtained after reducing the weights on the edges incident to the high degree vertices. We refer to~\cite[Section~1.4]{levina2018concentration} for more details. This extra technical step is not necessary when $\bar{d} = \Omega(\log n)$. 
Moreover, concentration also occurs if we replace the adjacency matrix with the normalized Laplacian in Equation~\eqref{eq:SSL_solution_minimization_noisy}. In that case, we obtain a generalization of the Label Spreading algorithm \cite{zhou2004learning}, \cite[Chapter~11]{chapelle2006semi}.

In the following, the mean-field graph refers to the weighted graph formed by the expected adjacency matrix of a DC-SBM graph.
Furthermore, we assume without loss of generality that the first~$n/2$ nodes are in the first cluster and the last~$n/2$ are in the second cluster. Therefore,
$
 \E A
 = Z B Z^T
$
with 
$
B = \begin{pmatrix}
\pin & \pout \\ \pout & \pin
\end{pmatrix}
$
and
$
Z =  \begin{pmatrix}
1_{n/2} & 0_{n/2} \\
0_{n/2} & 1_{n/2}
\end{pmatrix} .
$
In particular, the coefficients~$\theta_i$ disappear because $\E \theta_i = 1$.
We consider the setting in which the diagonal elements of $\E A$ are not zeros. This accounts for modifying the definition of DC-SBM, where we can have self-loops with probability $\pin$. Nevertheless, we could set the diagonal elements of $\E A$ to zeros and our results would still hold at the expense of cumbersome expressions. 
Note that the matrix $\E A$ has two non-zero eigenvalues: $\bar{d} = n \frac{\pin + \pout}{2} $ and $\bar{ \alpha } = n \frac{\pin - \pout}{2} $.

\begin{proof}[Proof of Theorem~\ref{thm:bound_number_misclassified_nodes}]
 We prove the statement in three steps. We first show that the solution $\hX$ of the constrained linear system~\eqref{eq:constraint_linear_system} is concentrated around the solution $\bx$ of the same system for the mean-field model. 
 Then, we compute $\bx$ and show that we can retrieve the correct cluster assignment from it. We finally conclude with the derivation of the bound.
 
  (i) Similarly to \cite{avrachenkov2018mean} and  \cite{avrachenkov2019almost}, let us rewrite equation~\eqref{eq:SSL_solution_minimization_noisy} as a perturbation of a system of linear equations corresponding to the mean-field solution. We thus have
 \begin{equation*}
  \big( \E  \tilde{\cL} + \Delta \tilde{\cL} \big) \big( \bar{x} + \Delta x \big) =  \lambda s,
 \end{equation*}
 where $\tilde{\cL} = - A_\tau + \lambda \cP - \gamma_* I_n$, $\Delta x := \hX - \bx$ and $\Delta \tilde{\cL} := \tilde{\cL} - \E \tilde{\cL}$.

 We recall that a perturbation of a system of linear equations $ (A + \Delta A) (x + \Delta x) = b $ leads to the following sensitivity inequality (see e.g., \cite[Section 5.8]{horn_johnson_2012}): 
 \begin{equation*}
  \dfrac{\|\Delta x\|}{\|x\|} 
  \wle \kappa(A) \dfrac{ \| \Delta A \| }{ \| A \| },
 \end{equation*}
 where $\|.\|$ is the operator norm associated with a vector norm $\|.\|$ (we use the same notations for simplicity) and $\kappa(A) := \|A^{-1}\| \cdot \|A\|$ is the condition number.
 In our case, the above inequality can be rewritten as follows:
 \begin{equation}
 \label{eq:in_proof_sensitivity}
  \dfrac{ \left\| \hX- \bar{x} \right\| }{ \left\| \bar{x} \right\| } 
  \wle  \left\| \left( \E \: \tilde{\cL} \right)^{-1} \right\|  \cdot  \left\| \Delta \: \tilde{\cL} \right\|,
 \end{equation}
 employing the Euclidean vector norm and spectral operator norm.
 The spectral study of $\E \:  \tilde{\cL}$ (see Corollary~\ref{cor:spectrum_Ltilde} in Appendix~\ref{appendix:rank_2}) gives:
 \begin{align*}
  \left\| \left( \E \: \tilde{\cL}\right)^{-1} \right\| & \weq \dfrac{1}{ \min \big \{ |\lambda| : \lambda \in \mathrm{Sp} \big( \E \: \tilde{\cL} \big) \big \} } = \dfrac{1}{ - t_2^+ -\bar{ \gamma}_* },
 \end{align*}
 where $t_2^+$ is defined in Corollary~\ref{cor:spectrum_Ltilde} in Appendix~\ref{appendix:rank_2} and $\bar{\gamma}_*$ is the solution of Equation~\eqref{eq:explicit_secular_equation} for the mean-field model. Lemma~\ref{lemma:bounding_gamma_star_mf} in Appendix~\ref{appendix:estimation_gamma_starMF} leads to 
 \begin{align}
 \label{eq:norm_EL_inverted}
  \left\| \left( \E \: \tilde{\cL} \right)^{-1} \right\| & \wle \frac{1}{ \lambda + \balpha }.
 \end{align}

 The last ingredient needed is the concentration of the adjacency matrix around its expectation. We have
 \begin{align*}
  \Big\|\tilde{\cL} - \E \tilde{\cL} \Big\| \wle \left\| \left( \gamma_* - \bar{ \gamma }_* \right) I_n \right\| + \|  A - \E \: A \| 
  \wle \left| \: \gamma_* - \bar{ \gamma }_* \: \right| + \|  A - \E \: A \|.
 \end{align*}
 Proposition~\ref{prop:concentration_gamma_*} in Appendix~\ref{appendix:concentration_gammastar} shows that
 \begin{align*}
  \left| \: \gamma_* - \bgamma_* \: \right| 
  & \wle 
  \left( 1 + \frac{\left( \bar{\alpha} + \lambda \right)^3 }{ 2 \sqrt{\eta_1+\eta_0} ( \eta_1 - \eta_0 )  \bar{\alpha}^2 \lambda  } \right) \sqrt{ \bar{d} }.
 \end{align*}
 Moreover, when $d = \Omega(\log n)$, it is shown in \cite{feige2005spectral} that for every $r>0$ there exists a constant $C'$ such that $\left\|  A - \E \: A \right\| \wle C' \sqrt{\bar{d}}$ holds with probability at least $1-n^{-r}$. 
 If $\bar{d} = o(\log n)$, the same result holds with a proper pre-processing on $A$, and 
 we refer the reader to \cite{levina2018concentration} for more details. We will omit this extra step in the proof to keep notations short. Using this concentration bound, we have
 \begin{align*}
  \Big\|\tilde{\cL} - \E \tilde{\cL} \Big\| & \wle 
  \left( C' + \frac{27 \left( \bar{\alpha} + \lambda \right)^3 }{ \sqrt{2} \sqrt{\eta_1+\eta_0} ( \eta_1 - \eta_0 )  \bar{\alpha}^2 \lambda  } \right) \sqrt{ \bar{d} } \\
  & \wle 
  \left( C' + \frac{27}{\sqrt{2}} \right)
  \frac{(\lambda+\bar{\alpha})^3 }{ \bar{\alpha}^2 \lambda }
  \frac{\sqrt{\bar{d} } }{ \sqrt{\eta_1+\eta_0} \left( \eta_1-\eta_0 \right)}.
 \end{align*}
 for some constant $C'$.
 Let $C = C' + \frac{27}{\sqrt{2}}$.
 By combining the above with inequality~\eqref{eq:norm_EL_inverted}, the inequality~\eqref{eq:in_proof_sensitivity} becomes
 \begin{equation}
 \label{eq:concentration_solution_around_mean_field}
  \dfrac{ \left\| \hX - \bx \right\| }{ \left\| \bx \right\| } 
  \wle C \, \frac{(\lambda+\bar{\alpha})^2 }{ \bar{\alpha}^2 \lambda }
  \frac{\sqrt{\bar{d} } }{ \sqrt{\eta_1+\eta_0} \left( \eta_1-\eta_0 \right)}.
 \end{equation}

 (ii) Node $i$ in the mean-field model is correctly classified by decision rule~\eqref{detection_criteria} if the sign of~$\bx_i$ equals the sign of~$Z_i$.
 Corollary~\ref{cor:recovery_with_the_mean_field} in Appendix~\ref{appendix:solution_mean_field} shows that this is indeed the case for the unlabeled nodes.
 
 (iii) Finally, for an unlabeled node $i$ to be correctly classified, the node's value~$\hX_{i}$ should be close enough to its mean-field value $\bx_{i}$.
 In particular, part (ii) shows that if $|\hX_{i} - \bx_{i}|$ is smaller than some non-vanishing constant $\beta$, then an unlabeled node $i$ will be correctly classified. An unlabeled node $i$ is said to be $\beta$-bad if $\left| \hX_i - \bx_i \right| > \beta$. We denote by $S_{\beta}$ the set of $\beta$-bad nodes.
The nodes that are not $\beta$-bad are a.s.~correctly classified, and thus $ L \left( \hZ_u, Z_u \right) \le \frac{ | S_\beta | }{ n } $.
		
 From 
 $\left\| \hX - \bx \right\|^2 
 \geq \sum\limits_{ i \in S_{\beta} } \left| \hX_i - \bx_i \right|^2$, it follows that $\left\| \hX - \bx \right\|^2 \geq | S_{\beta} | \times \beta^2$.
 Thus, using inequality~\eqref{eq:concentration_solution_around_mean_field} and the norm constraint $\left\| \bx \right\|^2 = n$, we have with probability at least $1-n^{-r}$, 
 \begin{align*}
 \left| S_{\beta} \right| & \wle \dfrac{1}{\beta^2} \left( \frac{C}{\eta_1 - \eta_0} \frac{ \balpha + \lambda }{ \bar{\alpha} \lambda} \sqrt{ \bd } \right)^2 n,
 \end{align*}
 for some constant $C$. We end the proof by noticing that $\frac{ \bd }{ \balpha } = \frac{\pin + \pout}{\pin - \pout}$.
 \end{proof}

 \begin{corollary}
 [Almost exact recovery in the diverging degree regime]
 \label{cor:almost_exact_recovery}
 Consider a DC-SBM such that $\bd \gg 1$, $\frac{\pin + \pout}{\pin - \pout} = O(1)$, $\sqrt{\eta_0+\eta_1}(\eta_1 - \eta_0)  \gg \frac{1}{\sqrt{ \bar{d} }}$, and $\eta_0 n \sqrt{\eta_0+\eta_1} \ll \lambda$.
 Suppose that $\tau > \pout$ and $\lambda \gtrsim \bar{ \alpha }$.
 Then, Algorithm~\ref{algo:SSL-SC-regularized_adjacency_matrix} correctly classifies almost all the unlabeled nodes.
 \end{corollary}
	
 \begin{proof}
 With the corollary's assumptions 
 $(\eta_1-\eta_0)^2 \bar{d} \rightarrow +\infty$ and $\frac{ \bar{ \alpha } + \lambda}{\lambda} = O(1)$, by Theorem~\ref{thm:bound_number_misclassified_nodes} the fraction of misclassified nodes is of the order $o(1)$.
 \end{proof}

 The quantity $(\eta_1 - \eta_0)n$ is the expected difference between the number of nodes correctly labeled and the number of nodes wrongly labeled by the oracle. 
 In particular, Corollary~\ref{cor:almost_exact_recovery} allows for a sub-linear number of labeled nodes since $\eta_0$ and $\eta_1$ can go to zero.
 

 \begin{corollary}[Detection in the constant degree regime]\label{cor:detection} Consider a DC-SBM such that $\pin = \frac{\cin}{n}$ and $\pout = \frac{\cout}{n}$ where $\cin, \cout$ are constants. Suppose that $\sqrt{\eta_0+\eta_1}(\eta_1 - \eta_0)$ is a non-zero constant, and let $\tau > 2\pout$ and $\lambda \gesim 1$. Assume furthermore that $\eta_0 n \sqrt{\eta_0+\eta_1} \ll \lambda$. Then, for $\frac{(\cin - \cout)^2}{\cin + \cout}$ bigger than some constant, w.h.p. Algorithm~\ref{algo:SSL-SC-regularized_adjacency_matrix} performs better than a random guess.
 \end{corollary}
	
 \begin{proof} 
  According to Theorem~\ref{thm:bound_number_misclassified_nodes}, the fraction of misclassified nodes is smaller than $\frac{1}{2}$ when $\frac{(\cin - \cout)^2}{\cin + \cout}$ is larger than 
  $\frac{4C}{(\eta_1 - \eta_0)^2} \left( \frac{ \bar{ \alpha } + \lambda}{\lambda} \right)^2$, which is indeed lower-bounded by a constant.
 \end{proof}
	
 The quantity $\frac{(\cin - \cout)^2}{\cin + \cout}$ can be interpreted as the signal-to-noise ratio. It is unfortunate that Corollary~\ref{cor:detection} does not allow us to control the constant in the statement of the corollary. This constant comes from the concentration of the adjacency matrix. Similar remarks were made in~\cite{levina2018concentration} for the analysis of spectral clustering in the constant degree regime for SBMs graph.

\section{Numerical experiments}
\label{section:numerical_results}

 This section presents numerical experiments both on simulated data sets generated from DC-SBMs and on real networks. In particular, we discuss the impact of the oracle mistakes (defined by the ratio $\frac{\eta_0}{\eta_0+\eta_1}$) on the performance of the algorithms.
 The code for the numerical experiments is available on GitHub at \url{https://github.com/mdreveton/ssl-sbm}.
 
\subsection{Synthetic data sets}

\paragraph{Choice of $\lambda$ and $\tau$}
Let us denote by $\sigma_1$ and $\sigma_2$ the largest and second largest eigenvalues of~$A$. We choose $\tau = \frac{4}{n} (\sigma_1+\sigma_2)$ and $\lambda =  \frac{ \log \frac{\eta_1}{\eta_0} }{ \log \frac{\sigma_1 + \sigma_2}{\sigma_1-\sigma_2} }$, if $\eta_0 \not=0$, and $\lambda = \frac{ \log \left( n \eta_1 \right) }{ \log \frac{\sigma_1 + \sigma_2}{\sigma_1-\sigma_2} }$, otherwise.
The heuristic for this choice is as follows. For a SBM graph, we have $\sigma_1 \approx \frac{n}{2}\left(\pin + \pout \right)$ and $\sigma_2 \approx \frac{n}{2} ( \pin - \pout)$, hence $ \frac{4}{n} (\sigma_1+\sigma_2) =  2 \pin > \pout $, and $\tau$ satisfies the condition of Theorem~\ref{thm:bound_number_misclassified_nodes}.
For $\lambda$, we have $ \frac{ \log \frac{\eta_1}{\eta_0} }{ \log \frac{\sigma_1 + \sigma_2}{\sigma_1-\sigma_2} } \approx \frac{ \log \frac{\eta_1}{\eta_0} }{ \log \frac{\pin}{\pout} }$, which is indeed close to the expression of $\lambda$ derived in Corollary~\ref{corollary:MAP_SBM} if $\pin, \pout = o(1)$.

\paragraph{Choice of relaxation}

We first compare the choice of the constraint in the continuous relaxation~\eqref{eq:MAP_relaxed}. Specifically, we compare the choice $\sum_i x_i^2 = n$ (we refer to it as \textit{standard relaxation}) versus $\sum_i d_i x_i^2 = 2|E|$ (we refer to it as \textit{degree-normalized relaxation}). This leads to two versions of Algorithm~\ref{algo:SSL-SC-regularized_adjacency_matrix}, whose cost obtained on SBMs graph with a noisy oracle is presented in Figure~\ref{fig:relaxatio_choice}. In particular, we observe that the normalized choice leads to a smaller cost. 
Therefore, in the following we will only consider the version of Algorithm~\ref{algo:SSL-SC-regularized_adjacency_matrix} solving the relaxed problem~\eqref{eq:MAP_relaxed} with constraint $\sum_{i} d_i x_i^2 = 2 |E|$ instead of $\sum_i x_i^2 = n$, as it gives better numerical results.

\begin{figure}[!ht]
\centering
\begin{subfigure}[b]{0.32\textwidth}
 \centering
 \includegraphics[width=\textwidth]{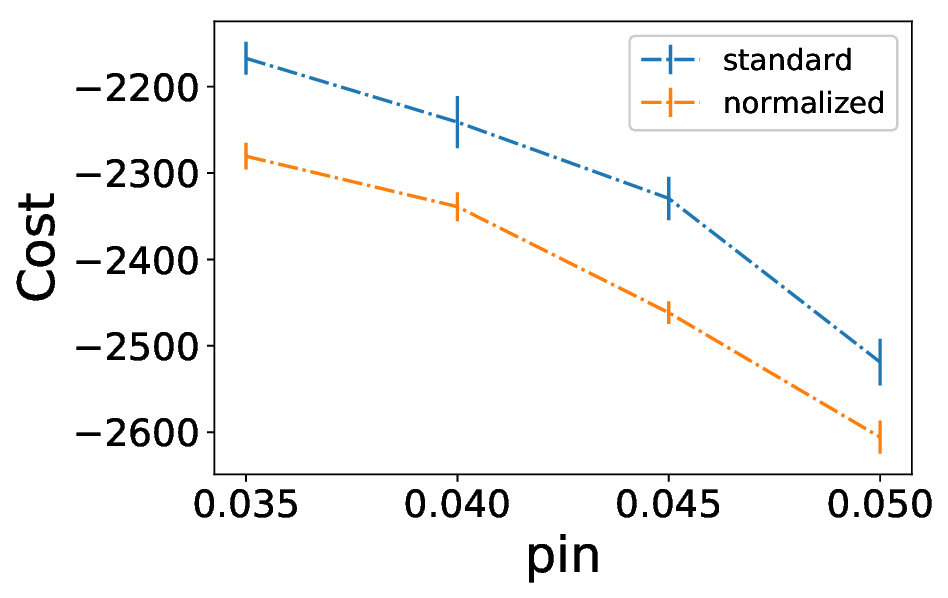}
\end{subfigure}
\caption{Cost in Algorithm~\ref{algo:SSL-SC-regularized_adjacency_matrix} with the standard and normalized versions of the constraint, on $50$ realizations of SBM with $n = 500$, $\pout = 0.03$ and $50$ labeled nodes with $10\%$ noise.}
\label{fig:relaxatio_choice}
\end{figure}

\paragraph{Experiments on synthetic graphs}

We first consider clustering on DC-SBM. We set $n = 2000$, $\pin = 0.04$ and $\pout = 0.02$. We consider three scenarios.
\begin{itemize}
 \item In Figure~\ref{Fig:accuracy_sbm} we consider a standard SBM ($\theta_i= 1$ for all $i$);
 \item In Figure~\ref{Fig:accuracy_normal} we generate $\theta_i$ according to $|\cN( 0, \sigma^2)| + 1 - \sigma \sqrt{2/\pi} $ where $| \cN(0, \sigma^2) |$ denotes the absolute value of a normal random variable with mean $0$ and variance $\sigma^2$. We take $\sigma = 0.25$. Note that this definition enforces $\E \theta_i = 1$. 
 \item In Figure~\ref{Fig:accuracy_pareto} we generate  $\theta_i$ from Pareto distribution with density function $f(x) = \frac{a m^a}{x^{a+1} } 1(x \geq m)$ with $a = 3$ and $m = 2/3$ (chosen such that $\E \theta_i = 1$).
\end{itemize}
We compare the performance of Algorithm~\ref{algo:SSL-SC-regularized_adjacency_matrix} with that of the algorithm of~\cite{Mai_Couillet_2020} (referred to as \textit{Centered similarities}) and the \textit{Poisson learning} algorithm described in~\cite{Calder_Cook_Thorpe_Slepcev_2020_poisson}.
We chose these two algorithms as references since they perform very well on real data sets and are designed to avoid flat solutions.
Results are shown in Figure~\ref{fig:evolution_accuracy_oracle_noise_sbm}.
We observe that when the oracle noise is low, the performance of Algorithm~\ref{algo:SSL-SC-regularized_adjacency_matrix} is comparable to \textit{Centered similarities}. But, when the noise becomes non-negligible, the performance of \textit{Centered similarities} deteriorates, while the accuracy of Algorithm~\ref{algo:SSL-SC-regularized_adjacency_matrix} remains high. We notice that \textit{Poisson learning} gives poor results on synthetic data sets.

\begin{figure}[!ht]
	\centering
	\begin{subfigure}[b]{0.32\textwidth}
		\centering
		\includegraphics[width=\textwidth]{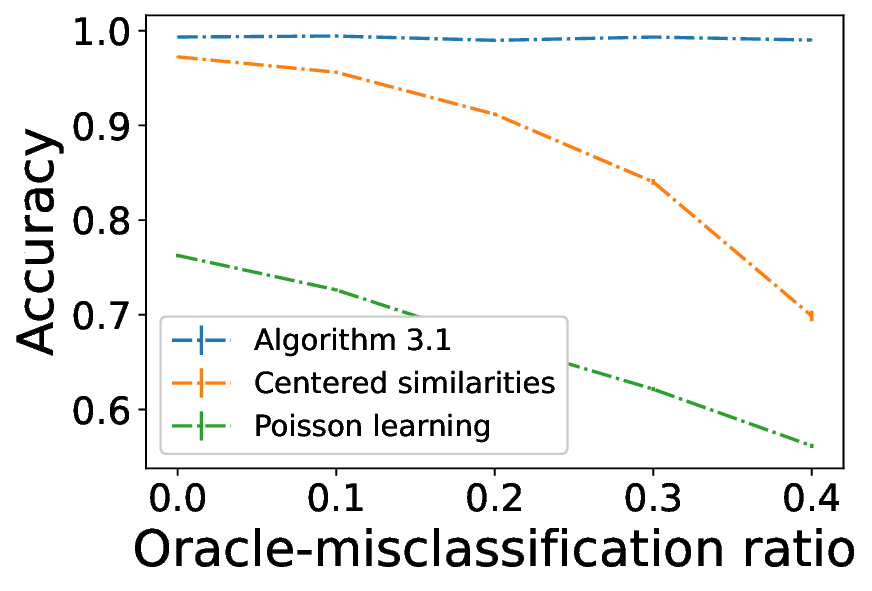}
		\caption{SBM}
		\label{Fig:accuracy_sbm}
	\end{subfigure}
	\hfill
	\begin{subfigure}[b]{0.32\textwidth}
		\centering
	    \includegraphics[width=\textwidth]{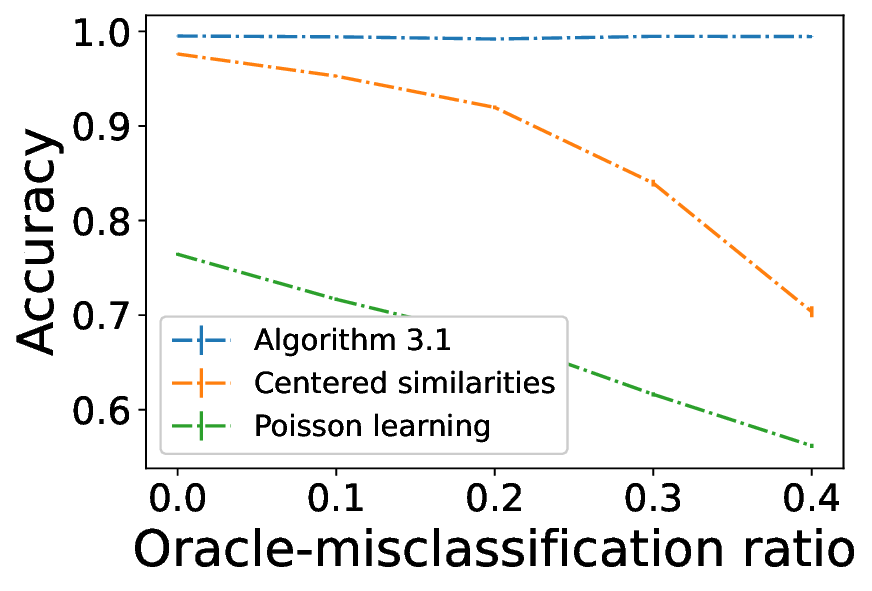}
		\caption{Normal Degree}
		\label{Fig:accuracy_normal}
	\end{subfigure}
	\hfill
	\begin{subfigure}[b]{0.32\textwidth}
		\centering
		\includegraphics[width=\textwidth]{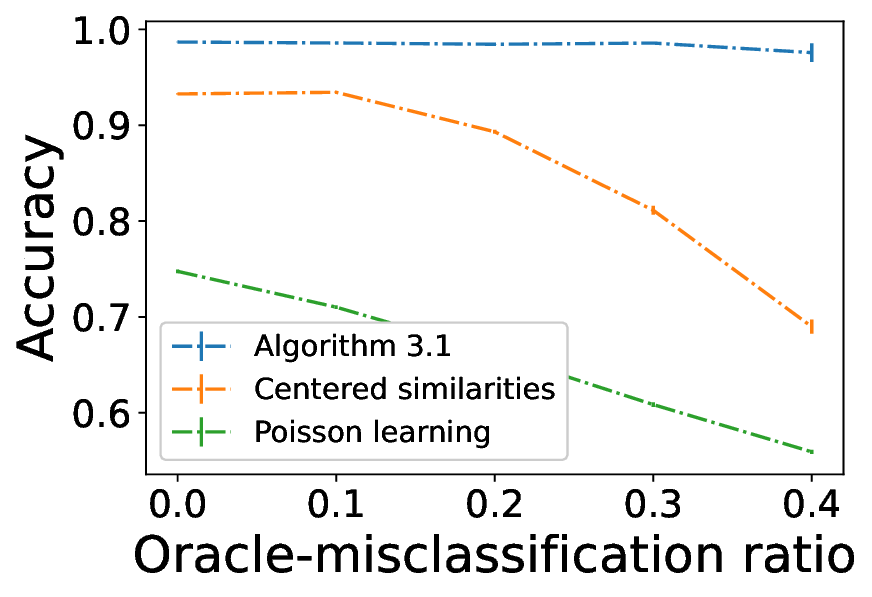}
		\caption{Pareto Degree}
		\label{Fig:accuracy_pareto}
	\end{subfigure}
	\caption{Average accuracy obtained by different semi-supervised clustering methods on DC-SBM graphs, with $n = 2000$, $\pin = 0.04$, and $\pout = 0.02$ with different distributions for $\theta$. The number of labeled nodes is equal to 40.
	Accuracies are computed on the unlabeled nodes, and are averaged over 100~realisations; the error bars show the standard error.
	}
	\label{fig:evolution_accuracy_oracle_noise_sbm}
    \end{figure}

\subsection{Experiments on real data}

We next use real data to show that even if real networks are not generated by the degree-corrected stochastic block model, Algorithm~\ref{algo:SSL-SC-regularized_adjacency_matrix} still performs well.

\paragraph{MNIST}
As a real-life example, we perform simulations on the standard MNIST data set~\cite{mnist}. As preprocessing, we select $1000$ images corresponding to two digits and compute the $k$-nearest-neighbors graph (we take $k = 8$) with Gaussian weights $w_{ij} = \exp\left( - \|x_i - x_j\|^2 / s_i^2 \right)$ where $x_i$ represents the data for image $i$ and $s_i$ is the average distance between $x_i$ and its $K$-nearest neighbors. 
Figure~\ref{fig:mnist} gives accuracy for different digit pairs. While the performance of \textit{Poisson learning} is excellent, it can suffer from the oracle noise. On the other hand, the accuracy of Algorithm~\ref{algo:SSL-SC-regularized_adjacency_matrix} remains unchanged.

\begin{figure}[!ht]
\centering
\begin{subfigure}[b]{0.32\textwidth}
 \centering
 \includegraphics[width=\textwidth]{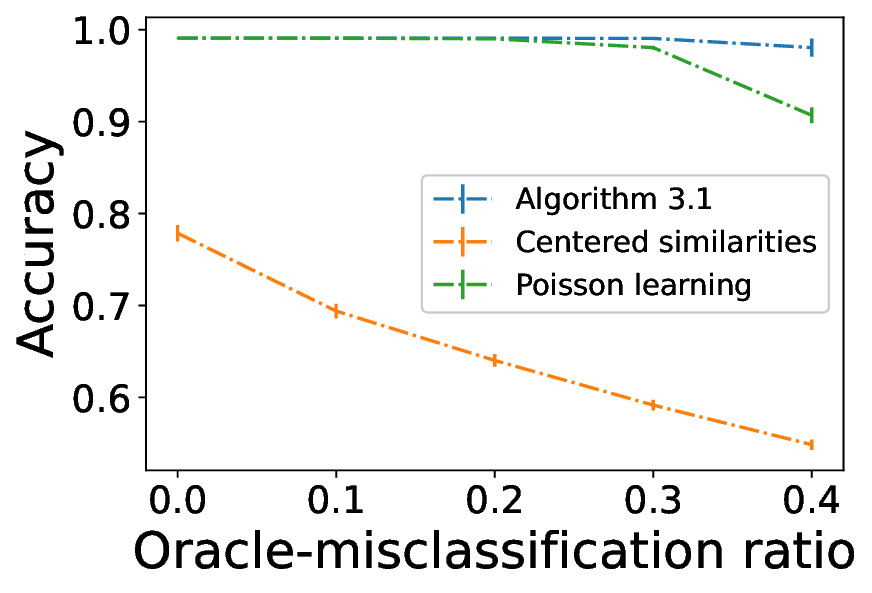}
 \caption{Digits (2,4).}
\end{subfigure}
\hfill
\begin{subfigure}[b]{0.32\textwidth}
 \centering
 \includegraphics[width=\textwidth]{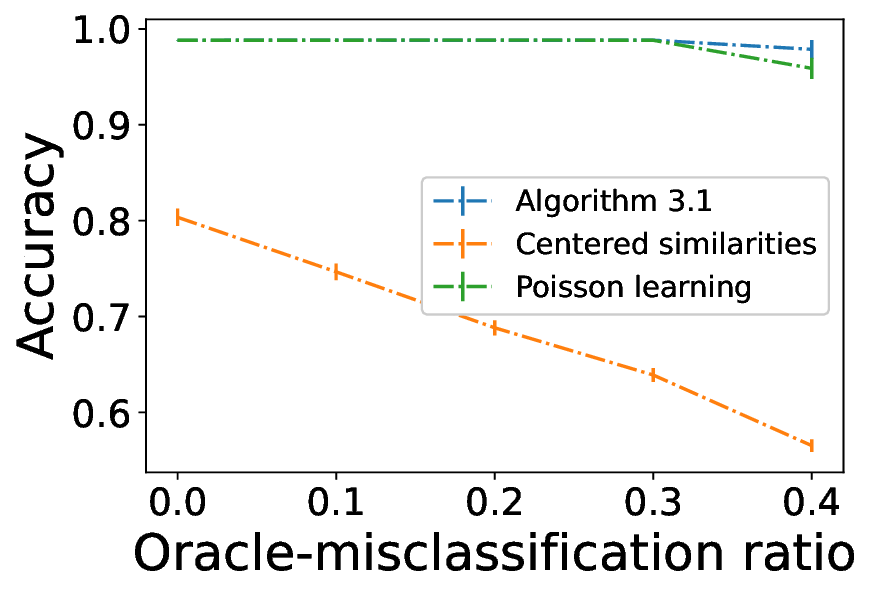}
 \caption{Digits (3,6).}
\end{subfigure}
\hfill
\begin{subfigure}[b]{0.32\textwidth}
 \centering
 \includegraphics[width=\textwidth]{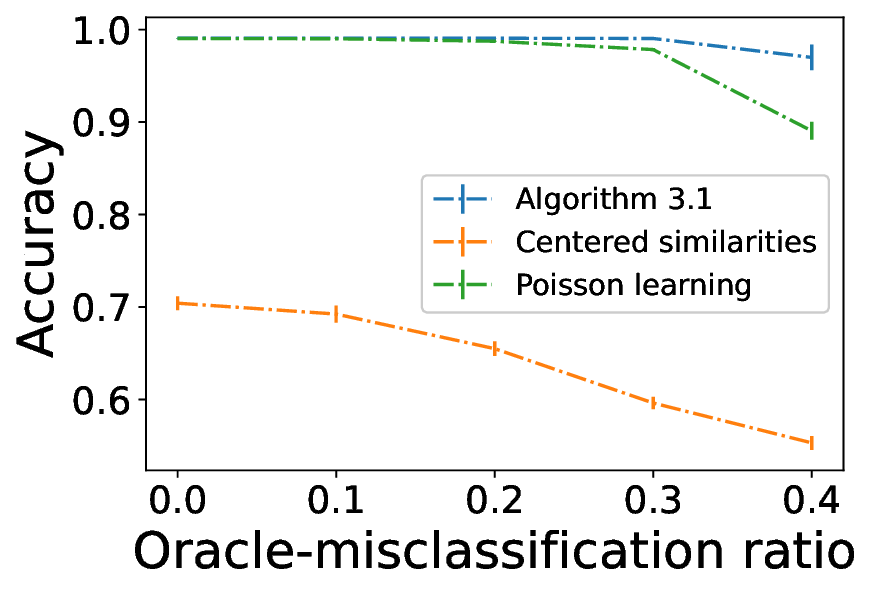}
 \caption{Digits (7,8).}
\end{subfigure}
\caption{Average accuracy obtained on a subset of the MNIST data set by different semi-supervised algorithms as a function of the oracle-misclassification ratio
, when the number of labeled nodes is equal to $10$. Accuracy is averaged over $100$ random realizations, and the error bars show the standard error.
}
\label{fig:mnist}
\end{figure}

To further highlight the influence of the noise, we plot in  Figure~\ref{fig:mnist_boxplots} the accuracy obtained by the three algorithms on the unlabeled nodes, the correctly labeled nodes, and the wrongly labeled nodes. 
We observe that the hard constraint $X_\ell = s_\ell$ imposed by \textit{Centered similarities} forces the correctly labeled nodes to be correctly classified. In contrast, the wrongly labeled nodes are not classified much better than a random guess. This heavily penalizes the unlabeled nodes' accuracy in an extremely noisy setting.
On the contrary, Algorithm~\ref{algo:SSL-SC-regularized_adjacency_matrix} allows for a smoother recovery: the unlabeled, correctly labeled, and wrongly labeled nodes have roughly the same classification accuracy. While some correctly labeled nodes are misclassified, many wrongly labeled nodes become correctly classified, and the unlabeled nodes are better recovered.
Finally, \textit{Poisson learning} shows a performance somewhere in between these two extreme cases: its accuracy on the unlabeled nodes is excellent, but it fails at correctly classifying the erroneously labeled nodes.

\begin{figure}[!ht]
\centering
\begin{subfigure}[b]{0.32\textwidth}
 \centering
 \includegraphics[width=\textwidth]{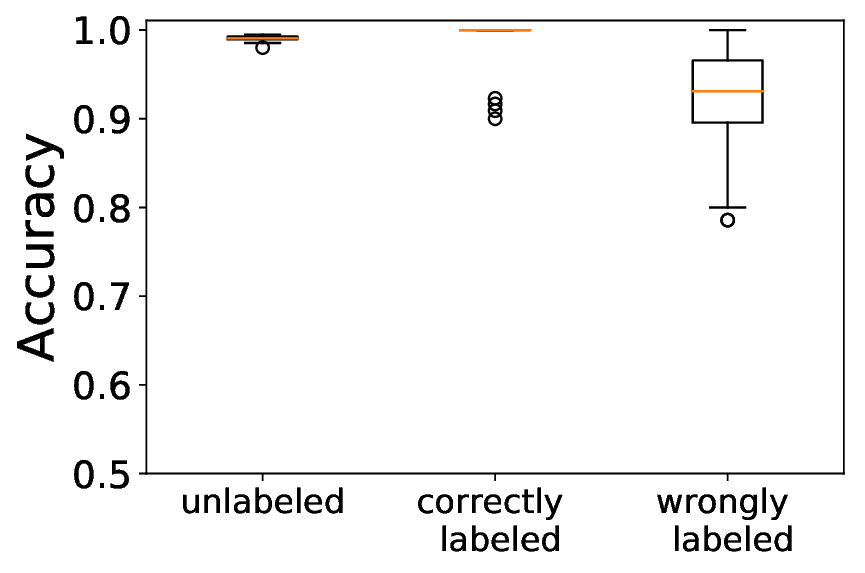}
 \caption{Algorithm~\ref{algo:SSL-SC-regularized_adjacency_matrix}.}
\end{subfigure}
\hfill
\begin{subfigure}[b]{0.32\textwidth}
 \centering
 \includegraphics[width=\textwidth]{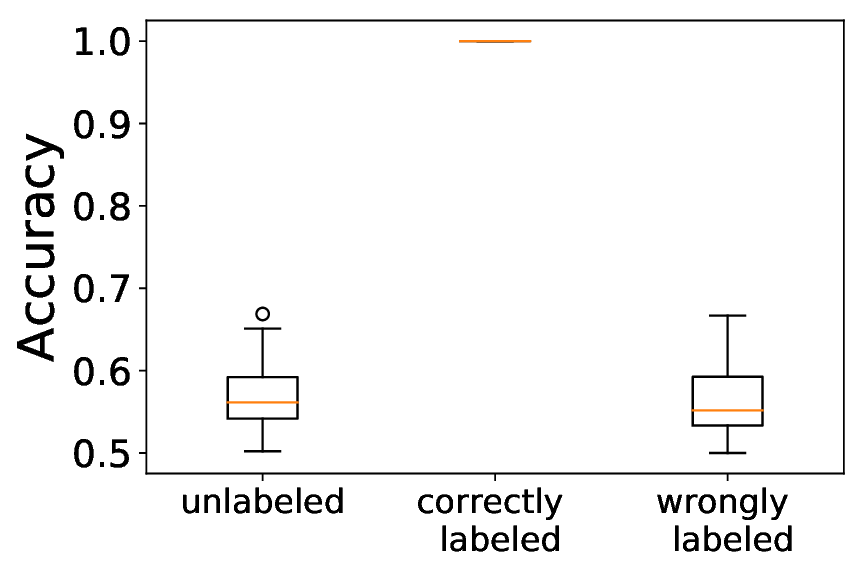}
 \caption{\textit{Centered Similarities}.}
\end{subfigure}
\begin{subfigure}[b]{0.32\textwidth}
 \centering
 \includegraphics[width=\textwidth]{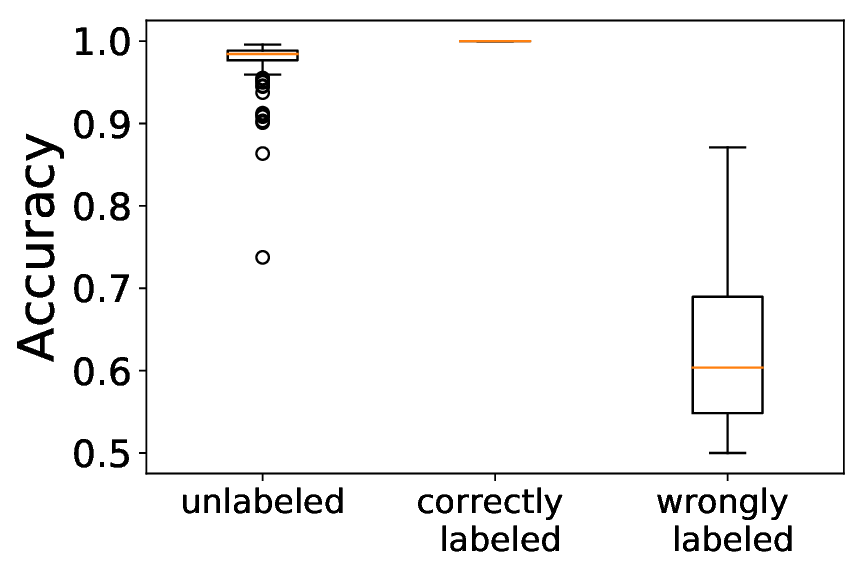}
 \caption{\textit{Poisson learning}.}
\end{subfigure}
\caption{Average accuracy obtained on the unlabeled, correctly labeled, and wrongly labeled nodes by the oracle. Simulations are done on the 1000 digits (2,4). The noisy oracle correctly classifies 24 nodes and misclassifies 16 nodes, and the boxplots show $100$ realizations.
}
\label{fig:mnist_boxplots}
\end{figure}

\paragraph{Common benchmark networks}

Finally, we perform simulations on three benchmark networks: \textit{Political Blogs}, \textit{LiveJournal}, and \textit{DBLP}. These networks are commonly used for graph clustering since the “ground truth” clusters are known. For \textit{LiveJournal} and \textit{DBLP}, we consider only the two largest clusters. The dimension of the data sets is given in Table~\ref{tab:real_networks_stat} and the performances of semi-supervised algorithms in Figure~\ref{fig:real_networks}. We observe that Algorithm~\ref{algo:SSL-SC-regularized_adjacency_matrix} and \textit{Poisson learning} outperform \textit{Centered similarities} and can still achieve good accuracy even in the presence of noise in labeled data.

\vspace{2cm}

\begin{table}[!ht]
    \centering
    \begin{tabular}{c c c c }
        Data set & $n_1$ & $n_2$ & $|E|$ \\ \hline
        \textit{Political Blogs}~\cite{adamic2005political} & 636 & 586 & 16,717 \\
        \textit{LiveJournal}~\cite{yang2015defining} & 1426 & 1340 & 24,138 \\
        \textit{DBLP}~\cite{yang2015defining} & 7373 & 5953 & 34,281 \\ \hline 
    \end{tabular}
    \caption{Parameters of the real data sets. $n_1$ (resp., $n_2$) corresponds to the size of the first (resp., second) cluster, and $|E|$ is the number of edges of the network.}
    \label{tab:real_networks_stat}
\end{table}

\begin{figure}[!ht]
\centering
\begin{subfigure}[b]{0.32\textwidth}
 \centering
 \includegraphics[width=\textwidth]{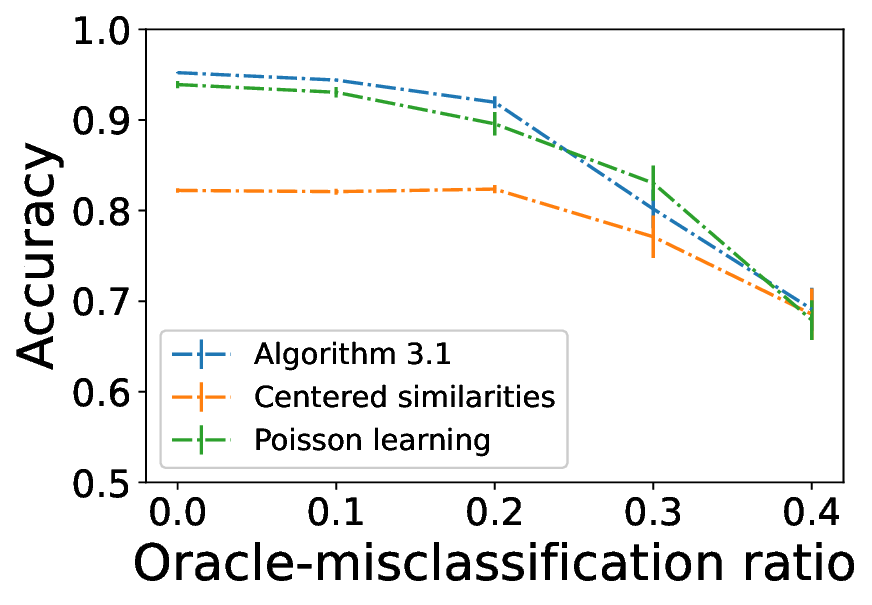}
 \caption{\textit{Political Blogs}}
\end{subfigure}
\hfill
\begin{subfigure}[b]{0.32\textwidth}
 \centering
 \includegraphics[width=\textwidth]{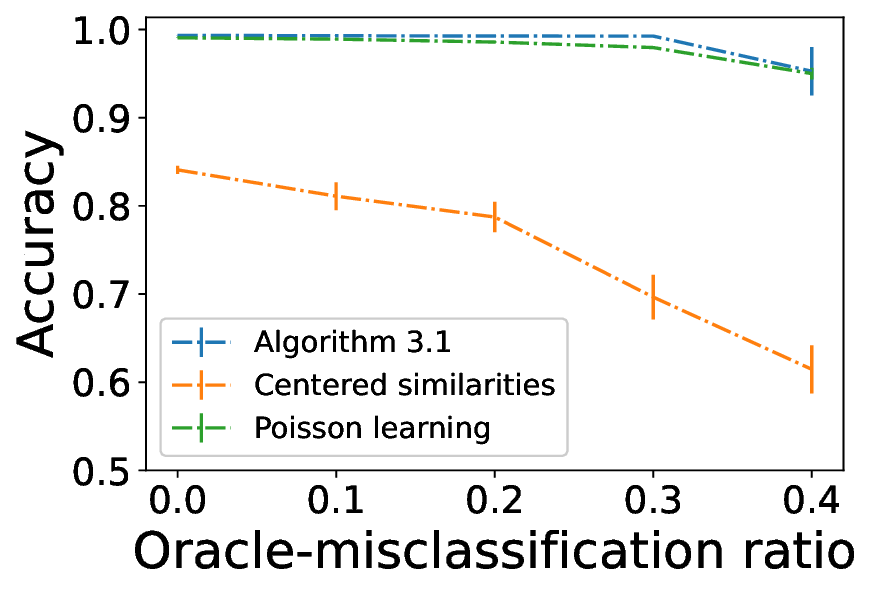}
 \caption{\textit{LiveJournal}}
\end{subfigure}
\hfill
\begin{subfigure}[b]{0.32\textwidth}
 \centering
 \includegraphics[width=\textwidth]{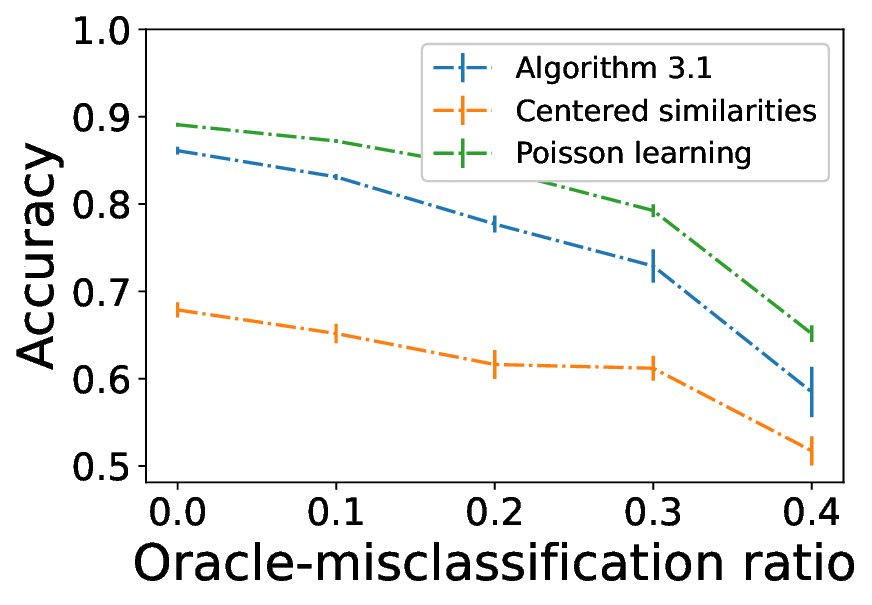}
 \caption{\textit{DBLP}}
\end{subfigure}
\caption{Average accuracy obtained on real networks by different semi-supervised algorithms as a function of the oracle-misclassification ratio. The number of labeled nodes is 30 for \textit{Political Blogs} and \textit{LiveJournal}, and $100$ for \textit{DBLP}. Accuracy is averaged over 50 random realizations, and the error bars show the standard error.
}
\label{fig:real_networks}
\end{figure}

\bibliographystyle{plain}
\bibliography{main.bib}

\appendix
\section{Derivation of the MAP}
\label{appendix_section_derivation_MAP}

\begin{proof}[Proof of Theorem~\ref{thm:MAP_dcsbm}]
 Bayes' formula gives
 $
 \pr(z \cond A, s) \propto \pr( A  \cond z, s  ) \ \pr( z \cond s ),
 $
 where the proportionality symbol hides $\pr(A \cond s)$-term independent of $z$.
 
The likelihood term can be rewritten as follows:
\begin{align*}
 \pr( A \cond  z, s ) \weq \pr( A \cond z)
 \, \propto \, \prod_{ \substack{ i<j \\ z_i = z_j } } \left( \frac{\pin}{\pout} \frac{1-\theta_i \theta_j \pout}{1 - \theta_i \theta_j \pin} \right)^{a_{ij}} \left( \frac{1- \theta_i \theta_j \pin}{ 1 - \theta_i \theta_j \pout } \right),
\end{align*}
where the proportionality hides a constant $ C = \prod\limits_{i<j} \left( \frac{\theta_i \theta_j \pout}{1 - \theta_i \theta_j \pout} \right)^{a_{ij}} \left( 1 -\theta_i \theta_j \pout \right)$ independent of~$z$. Hence,
\begin{align}
 \log \pr\left( A \cond z, s \right) & \weq \log C + \frac{1}{2} \sum_{i,j} 1( z_i \not= z_j ) \left(  \left( f^{(1)}_{ij} - f^{(0)}_{ij} \right) a_{ij}  + f^{(0)}_{ij} \right) \nonumber \\
 & \weq \log C + \frac{1}{2} \sum_{i,j= 1}^n \frac{1 - z_i z_j}{2} \left(  \left( f^{(1)}_{ij} - f^{(0)}_{ij} \right) a_{ij}  + f^{(0)}_{ij} \right) \nonumber \\
 & \weq \log C' - \frac{1}{4} x^T M x 
 \label{eq:in_proof_likelihood}.
\end{align}
for some constant $C'$ and $M = (F_1 - F_0) \odot A + F_0 $.
	
The oracle information, given by the term $\pr( z \cond s )$, is equal to
\begin{align}
 \pr(z \cond s) & \weq \prod_{i=1}^n \frac{\pr(s_i \cond z_i) }{\pr(s_i)} \pr( z_i ) \nonumber \\
 & \weq \left( \dfrac{\eta_1}{\eta_1 + \eta_0 }\right)^{ \big|\{ i\in \ell \colon z_i = s_i \} \big| } \ \left( \dfrac{\eta_0}{\eta_1 + \eta_0}\right)^{ \big|\{ i\in \ell  \colon z_{i} \not= s_{i} \} \big| } \ \left( \dfrac{1}{2} \right)^{ n }  \nonumber \\
 & \weq \left( \dfrac{\eta_0}{\eta_1}\right)^{  \big|\{ i \in \ell \, \colon z_{i} \not= s_{i} \} \big| } \left( \dfrac{\eta_1}{\eta_1+\eta_0} \right)^{  \big| \ell \big| } \left( \dfrac{1}{2} \right)^{ n  }, \label{eq:prob_sigma_S_simplified}
\end{align}
where we used $ \big|\{ i \in \ell \colon z_{i}  = s_{i} \} \big| +  \big|\{ i \in \ell \colon z_{i} \not= s_{i} \} \big| =  \big| \ell \big|$ in the last line.
Noticing that 
\begin{align*}
 \left| \{ i \in \ell : z_{i} \not= s_{i}  \} \right| 
 \weq \frac{1}{4} \sum_{i=1}^n \left( \left( \cP z\right)_{i}  - s_{i} \right)^2
 \weq \frac14  \left(  \cP z - s  \right)^T \left(  \cP z - s  \right), 
\end{align*}
yields
\begin{align}
 \log \pr\left( z \cond s \right) 
 & \weq - \frac{1}{4} \log \left( \frac{\eta_1}{\eta_0} \right) \cdot 
 \left\| \cP z - s  \right\|^2
 +  C',
\label{eq:in_proof_prior}
\end{align} 
where $C'$ is a term independent of $z$.

If $\eta_0 \not= 0$, the combination of Equations~\eqref{eq:in_proof_likelihood} and~\eqref{eq:in_proof_prior} with Bayes' formula gives Expression~\eqref{eq:MAP_dcsbm_SSL_noise}.
If $\eta_0 = 0$, then from Equation~\eqref{eq:prob_sigma_S_simplified} the term $\pr( z \cond s)$ is non-zero (and constant) if and only if $z_{i} = s_{i}$ for every labeled node $i \in [\ell]$, and we obtain Expression~\eqref{eq:MAP_dcsbm_SSL_no_noise}.
\end{proof}

\begin{proof}[Proof of Corollary~\ref{corollary:MAP_SBM}]
 The proof follows from Theorem~\ref{thm:MAP_dcsbm} and the fact that $f_{ij}^{ (0) } = \log \frac{ 1-\pin }{ 1-\pout }$ and $f_{ij}^{ (1) } = \log \frac{ \pin }{ \pout }$.
\end{proof}

\section{Lemmas related to mean-field solution of the secular equation}

\subsection{Spectral study of a perturbed rank-2 matrix}
\label{appendix:rank_2}

\begin{lemma}[Matrix determinant lemma]
\label{lemma:matrix_determinant_lemma}
 Suppose $A \in \R^n$ is invertible, and let $U, V$ be two $n$ by $m$ matrices. Then $\det(A + U V^T) = \det A \det(I_m + V^T A^{-1} U)$.
\end{lemma}
\begin{proof}
 Take the determinant of 
	$
	\begin{pmatrix}
	A & -U \\ V^T & I
	\end{pmatrix} 
	=
	\begin{pmatrix}
	A & 0 \\
	V^T & I 
	\end{pmatrix} 
	.
	\begin{pmatrix}
	I & -A^{-1} U \\
	0 & I + V^T A^{-1} U
	\end{pmatrix}
	$
	and notice that 
	$
	\det \begin{pmatrix}
	A & -U \\ V^T & I
	\end{pmatrix}  = \det I \det\left( A  + U V^T \right)
	$
	by the Schur complement formula \cite[Section~0.8.5]{horn_johnson_2012}.
\end{proof}

\begin{proposition}
\label{prop:interesting_determinant}
 Let $M = Z B Z^T $, where $B = \begin{pmatrix}
 a & b \\ b & a
 \end{pmatrix}$ is a $2\times2$ matrix, and $Z = \begin{pmatrix}
 1_{n/2} & 0_{n/2} \\
 0_{n/2} & 1_{n/2}
 \end{pmatrix}$ is an $n\times2$ matrix. Let $m$ be an even number. We denote by $P_\cL$ the $n \times n$ diagonal matrix whose first $\frac{m}{2}$ and last $\frac{m}{2}$ diagonal elements are ones, all other elements being zeros. Then,
 $
 \det \Big( t I_n + \lambda P_\cL - M  \Big) \weq t^{n-m-2} (t+\lambda)^{m-2} (t-t_1^+) (t-t_1^-) (t-t_2^+) (t-t_2^-)
 $
 with
 \begin{align*}
  t_1^\pm & \weq  \dfrac{1}{2}\Bigg(  \frac{n}{2}(a+b) - \lambda \pm \sqrt{ \Big( \lambda + \frac{n}{2} (a+b) \Big)^2 - 2 (a+b) \lambda m }   \Bigg),  \\
  t_2^\pm & \weq  \dfrac{1}{2}\Bigg(  \frac{n}{2}(a-b) - \lambda \pm \sqrt{ \Big( \lambda + \frac{n}{2} (a-b) \Big)^2 - 2 (a-b) \lambda m }   \Bigg).
 \end{align*}
\end{proposition}

\begin{proof}For now, assume that $t \not = - \lambda$ and $t\not = 0$. Then, $t I_n + \lambda \, P_\cL $ is invertible, and by Lemma~\ref{lemma:matrix_determinant_lemma},
	\begin{align}
	\det \Big( t I_n + \lambda P_\cL - M  \Big)  & \weq \det( tI_n + \lambda P_\cL ) \det\Big(I_2 + Z^T (t I_n + \lambda P_\cL )^{-1} (-ZB) \Big) \nonumber \\
	& \weq  (t+\lambda)^{m} t^{n-m} \det\Big(I_2 - Z^T (t I_n + \lambda P_\cL )^{-1} ZB \Big) \label{eq:in_proof_after_determinant_lemma}.
	\end{align}
	Moreover,
	\begin{align*}
	\big( t I_n + \lambda \, P_\cL \big)^{-1} 
	\weq \dfrac{1}{t}(I_n - P_\cL) + \dfrac{1}{t+\lambda} P_\cL
	\weq \dfrac{1}{t} I_n - \dfrac{\lambda}{t (t+\lambda)} P_\cL.
	\end{align*}
	Therefore, we can write
	\begin{align*}
	Z^T \big( t I_n + \lambda \, P_\cL \big)^{-1} ZB 
	\weq \dfrac{1}{t} Z^T Z B - \dfrac{\lambda}{t(t+\lambda) } Z^T P_{\cL} ZB 
	 \weq \dfrac{1}{t} \dfrac{n}{2} B - \dfrac{\lambda}{t(t+\lambda)} \dfrac{m}{2} B 
	\weq x B,
	\end{align*}
	where $x := \dfrac{n}{2} \dfrac{1}{t(t+\lambda)} \left( t + \lambda \left( 1 - \dfrac{m}{n} \right) \right) $.
	Thus, a direct computation of the determinant gives
	\begin{align*}
	\det \Big(I_2 - Z^T \big( t I_n + \lambda \, P_\cL \big)^{-1} ZB \Big) \weq \Big(1-x(a+b)\Big) \Big(1-x(a-b)\Big).
	\end{align*}
	Going back to equation~\eqref{eq:in_proof_after_determinant_lemma},
	we can write
	\begin{align}\label{eq:in_proof_almost_over}
	\det \Big( t I_n + \lambda P_\cL - M  \Big) & \weq (t+ \lambda)^{m-2} t^{n-m-2} P_1(t) P_2(t),
	\end{align}
	with $P_1(t) = t(t+\lambda) - \frac{n}{2} (a+b) \big(t+\lambda(1-\frac{m}{n}) \big) $ and $P_2(t) = t(t+\lambda) - \frac{n}{2} (a-b) \big(t+\lambda(1-\frac{m}{n}) \big) $.
	Since $t \in \R \mapsto \det (t I_n + \lambda P_\cL - M )$ is continuous (even analytic), expression~\eqref{eq:in_proof_almost_over} is also valid for $t= 0$ and $t = - \lambda$ \cite{AFH2013}. 
	We end the proof by observing that
	\begin{align*}
	P_1(t)  \weq (t-t_1^+) (t-t_1^-)
	\qquad \text{ and } \qquad
	P_2(t) \weq (t-t_2^+) (t-t_2^-),
	\end{align*}
	where $t_1^\pm$ and $t_2^\pm$ are defined in the proposition's statement.
\end{proof}

\begin{corollary}
\label{cor:spectrum_Ltilde} 
Let $A$ be the adjacency matrix of a DC-SBM with $\pin>\pout>0$, and $s$ be the oracle information.
 Let  $\lambda, \tau >0$, and $\bar{d}_\tau = \frac{n}{2}\left( \pin + \pout \right) - n \tau$, $\bar{\alpha} = \frac{n}{2}\left( \pin - \pout \right)$.
 Let $A_\tau := A - \tau 1_n 1_n^T$ and $P_\cL$ be the diagonal matrix whose element $(P_\cL)_{ii}$ is $1$ if $s_i \not=0$, and $0$ otherwise.
 Then, the spectrum of  $\E \tilde{\cL} =- \E A_\tau + \lambda \cP - \gamma I_n $ is
 $
 \left\{ -\gamma - t_1^\pm; -\gamma - t_2^\pm; -\gamma; -\gamma + \lambda;  0  \right\},
 $
 where 
 \begin{align*}
 t_1^\pm & \weq  \dfrac{1}{2}\Bigg(  \bar{d}_\tau - \lambda \pm \sqrt{ \left( \lambda + \bar{d}_\tau \right)^2 - 4 \bar{d}_\tau \lambda \left( \eta_1 + \eta_0 \right) } \Bigg),  \\
 t_2^\pm & \weq \dfrac{1}{2} \Bigg( \bar{\alpha} - \lambda \pm \sqrt{ \Big( \lambda + \bar{\alpha} \Big)^2 - 4 \bar{\alpha} \lambda \left( \eta_1 + \eta_0 \right) }   \Bigg).
 \end{align*}
\end{corollary}

\begin{proof}
Let $M = \begin{pmatrix}
	\pin - \tau & \pout - \tau \\
	\pout - \tau & \pin - \tau
	\end{pmatrix}$ and $Z = \begin{pmatrix}
	1_{n/2} & 0_{n/2} \\
	0_{n/2} & 1_{n/2}
	\end{pmatrix}$. 
	Then, we notice that  
	$\E A_\tau = Z M Z^T$ and we can apply Proposition~\ref{prop:interesting_determinant} to compute the characteristic polynomial of $\E \tilde{\cL}$. For $x\in \R$,
	$
	\det \left( \E \tilde{\cL} - x I_n \right) \weq \det \Big( (-\gamma - x) I_n - \E A_\tau + \lambda \cP \Big) ,
	$
	whose roots are $-\gamma - t_1^{\pm}, -\gamma-t_2^{\pm}$, $-\gamma$, and $-\gamma + \lambda$.
\end{proof}

\subsection{Bounds for \texorpdfstring{$\bgamma_*$}{gamma*}}
\label{appendix:estimation_gamma_starMF}

\begin{lemma}
\label{lemma:bounding_gamma_star_mf}
Let $\bar{\gamma}_*$ be the solution of Equation~\eqref{eq:explicit_secular_equation} for the mean-field model.
Then,
\begin{align*}
    -\bar{\alpha} ( 1-2\eta_0) 
    \wle
    \bar{\gamma}_*
    \wle - \bar{ \alpha }.
\end{align*}
\end{lemma}

\begin{proof}
For $\lambda \ge 0$, we denote by $(\bar{x}_\lambda, \bar{\gamma}_*(\lambda) )$ the solution of the system~\eqref{eq:constraint_linear_system} on a mean-field DC-SBM.
The proof is in two steps. 
First, let us show that $\bgamma_*(0) = -\balpha$ and $\bgamma_*(\infty) = -\balpha(1-2\eta_0)$.
For $\lambda = 0$, the constrained linear system~\eqref{eq:constraint_linear_system} reduces to an eigenvalue problem, and hence $\bgamma_*(0)$ equals $-\alpha$, the smallest eigenvalue of $-\E A_{\tau}$.
Moreover, when $\lambda = \infty $, the hard constraint $x_{\ell} = \bs_{\ell}$ is enforced, and the system~\eqref{eq:constraint_linear_system} becomes
\begin{align*}
 \left\{
 \begin{array}{rl}
 (-\E A_{\tau} - \bgamma_*(\infty) I_n)_{uu} \bx_u 
 & \weq (\E A_{\tau})_{u\ell} \bs_\ell \\
 \bx_u^T \bx_u 
 & \weq n (1 - \eta_0 - \eta_1)
 \end{array}
 \right.
\end{align*}
and we verify by hand that $ \bgamma_*(\infty) = - \balpha(1-2\eta_0)$ together with $\bx_u = Z_u$ is indeed the solution.

Second, if we let $C_\lambda(x) = -x^T \E A_\tau x + \lambda ( \bs - \cP x)^T ( \bs - \cP x)$ be the cost function minimized in~\eqref{eq:MAP_relaxed}, then from Equation~\eqref{eq:constraint_linear_system} we have $\bgamma_*(\lambda_1) - \bgamma_*(\lambda_2) = C_{\lambda_1}(\bx_1) - C_{\lambda_2}(\bx_2) + \lambda_1 \bx_1^T \bs - \lambda_2 \bx_2^T  \bs$. Since $\lambda \mapsto C_\lambda(x)$ is increasing, then $\lambda_1 \le \lambda_2$ implies $C_{\lambda_1}(\bx_1) \le C_{\lambda_2}(\bx_2)$. Since $\bx_\lambda^T \bs \ge 0$ (if it was not the case, then $C_\lambda(-\bx_\lambda) \le C_\lambda(\bx_\lambda)$, and hence $\bx_\lambda \not= \argmin_{x \in \R^n} C_\lambda(x)$), we can conclude that $\bgamma_*(0) \le \bgamma_*(\lambda)$ and that $\bgamma_*(\lambda) \le \bgamma_*(\infty)$.
\end{proof}

\subsection{Concentration of \texorpdfstring{$\gamma_*$}{gamma*}}
\label{appendix:concentration_gammastar}

\begin{proposition}
\label{prop:concentration_gamma_*}
Let $\gamma_*$ and $\bar{\gamma}_*$ be the solutions of Equation~\eqref{eq:constraint_linear_system} for a DC-SBM and the mean-field DC-SBM, respectively. Then
    \[
    	    \left| \gamma_* - \bar{\gamma}_* \right| 
    	    \wle 
    	    \left( 1 + \frac{\left( \bar{\alpha} + \lambda \right)^3 }{ 2 \sqrt{\eta_1+\eta_0} ( \eta_1 - \eta_0 )  \bar{\alpha}^2 \lambda  } \right) \sqrt{ \bar{d} }.
    \]
\end{proposition}

\begin{proof}
 The gradient with respect to 
 $(\bdelta_1,...,\bdelta_n, \bb_1,...,\bb_n, \gamma)$
 of the left-hand-side of Equation~\eqref{eq:explicit_secular_equation} is equal to
	\[
	2 \sum_{i=1}^n \frac{ \bb_i }{ \bdelta_i - \bgamma }
	\left[ \frac{ \Delta b_i }{ \bdelta_i - \bgamma_* } - \frac{ \bb_i \Delta \delta_i }{ ( \bdelta_i - \bgamma_* )^2 }
	+ \frac{ \bb_i \Delta \gamma}{ ( \bdelta_i - \bgamma_* )^2 } \right].
	\]
	Thus, we have
	\[
	\Delta \gamma \sum_{i=1}^n 
	\frac{ \bb_i^2 }{ ( \bdelta_i - \bgamma_* )^3 }
	\weq
	\sum_{i=1}^n 
	\frac{ \bb_i^2 }{ ( \bdelta_i - \bgamma_*)^3 } \Delta \delta_i -
	\sum_{i=1}^n \frac{ \bb_i }{ ( \bdelta_i - \bgamma_* )^2 } \Delta b_i + o \left( \Delta \delta_i, \Delta b_i \right).
	\]
	Firstly, we see that for all $i \in [n]$,
	$\Delta \delta_i = \left| \delta_i - \bar{\delta}_i \right| \le \left\| A - \E A \right\| \le \bar{d}$ by the concentration of the adjacency matrix of a DC-SBM graph.
	Therefore, using this fact and $\bar{\gamma}_* \le \bar{\delta}_1 \le \bar{\delta}_2 \le \cdots \le \bar{\delta}_n$,
	\begin{align*}
	\Delta \gamma \weq  \left| \gamma_* - \bar{\gamma}_* \right| 
	& \wle \max_{i} \left| \delta_i - \bdelta_i \right|  
	+ \frac{ \max_{i} \frac{1}{ ( \bdelta_i - \bgamma_*)^2 } }{ \min_{i} \frac{1}{ ( \bdelta_i - \bgamma_* )^3 } }  \frac{ \sum_{i} | \bb_i | \cdot | b_i - \bb_i | }{ \sum_i \bb_i^2 } \\
	& \wle  \sqrt{ \bar{d} }  + \frac{ \max_i \left( \bdelta_i - \bgamma_* \right)^3 }{ \min_i \left( \bdelta_i - \bgamma_* \right)^2 } \frac{\sum_{i} |  \bb_i | \cdot | b_i - \bb_i |}{ \sum_i \bb_i^2 }.
	\end{align*}
 We notice that $\min_i | \bar{\delta}_i - \bar{\gamma}_* | = \bar{\delta}_1 - \bar{\gamma}_*$.
 By using Lemma~\ref{lemma:bounding_gamma_star_mf} and the expression of $\bar{\delta}_1$ given in Corollary~\ref{cor:spectrum_Ltilde}, we have
 \[
 \min_i | \bdelta_i - \bgamma_* | \wge \bar{\alpha} + \lambda.
 \]
 Similarly, $\max_i | \bar{\delta}_i - \bar{\gamma}_* | = \bar{\delta}_n - \bar{\gamma}_* = \bar{\delta}_n - \bar{\delta}_1 + \bar{\delta}_1 - \bar{\gamma}_*$. Corollary~\ref{cor:spectrum_Ltilde} implies $\bar{\delta}_n = \lambda$ and $\bar{\delta}_1 = \frac{1}{2} \left( \lambda - \bar{\alpha} - \sqrt{ \left(\lambda + \bar{\alpha}\right)^2 - 4\bar{\alpha} \lambda (\eta_0+\eta_1)}\right)$, thus $\bar{\delta}_n - \bar{\delta}_1 \le \bar{\alpha} + \lambda$. Hence, using Lemma~\ref{lemma:bounding_gamma_star_mf},
 \[
 \max_i | \bdelta_i - \bgamma_* |  
 \wle \frac{3}{2} \left( \balpha + \lambda \right).
 \]
 Therefore, we have
 \begin{align}
 \label{eq:in_proof_bounding_delta_gamma}
 \left| \gamma_* - \bar{\gamma}_* \right| 
 & \wle \sqrt{ \bar{d} }  + \frac{27}{8} ( \bar{\alpha} + \lambda ) \cdot \frac{\sum_{i} | \bb_i | \cdot | b_i - \bb_i |}{ \sum_i \bb_i^2 }.
 \end{align}
 The term $\frac{\sum_{i} | \bb_i | \cdot | b_i - \bb_i | }{ \sum_i \bb_i^2}$ can be bounded as follow.
 Let $\cI = \{ i \in [n] \colon \bb_i \not=0 \}$.
 Then
 \begin{align*}
 \sum_{i} | \bb_i | \cdot | b_i - \bb_i |
 & \wle \max_{i \in \cI} | b_i - \bb_i | \cdot \sum_{i \in \cI} \left| \bb_i \right|.
 \end{align*}
 Combining the Cauchy–Schwarz inequality
 \begin{align*}
 \left| b_i - \bb_i \right| 
 \weq \lambda \left| ( Q_{\cdot i} - \bar{Q}_{\cdot i})^T \bs \right|
 \wle \lambda \left\| Q_{\cdot i} - \bar{Q}_{\cdot i} \right\|_2 \cdot \| \bs \|,
 \end{align*}
 with the Davis-Kahan theorem~\cite{Yu_Wang_Samworth_2015}
 \begin{align*}
 \left\| Q_{\cdot i} - \bar{Q}_{\cdot i} \right\|_2 & \wle \frac{2^{3/2} \left\| A - \E A\right\| }{ \min \left\{ \bar{\delta}_i - \bar{\delta}_{i-1}, \bar{\delta}_{i+1} - \bar{\delta}_i \right\} },
 \end{align*}
 $\| \bs \| = \sqrt{(\eta_0+\eta_1)n}$, and the concentration of $A$ towards $\E A$, yields
 \begin{align*}
 \max_{i \in \cI} | b_i - \bb_i |
 \wle 
 \frac{\lambda \sqrt{ (\eta_0+\eta_1) n } }{
 \min_{i \in \cI } \left\{ \bdelta_i - \bdelta_{i-1}, \bdelta_{i+1} - \bar{\delta}_i \right\} }
 \cdot 2^{3/2} \sqrt{\bar{d}}.
 \end{align*}
 Using Lemma~\ref{lemma:estimations_di}, we see that $\cI = \{ i \in [n] : \delta_i \not\in \{0, t_1^-\} \}$. Combining it with Corollary~\ref{cor:spectrum_Ltilde}, gives
 \begin{align*}
    \min_{i \in \cI } \left\{ \bar{\delta}_i - \bar{\delta}_{i-1}, \bar{\delta}_{i+1} - \bar{\delta}_i \right\}
	& \weq \lambda + t_2^+ \\
	& \weq \frac{\alpha+\lambda}{2} \left( 1 - \sqrt{ 1 - 4 \frac{\alpha \lambda}{(\alpha+\lambda)^2} (\eta_0+\eta_1) } \right) \\
	& \wge \frac{\alpha \lambda}{\alpha + \lambda} (\eta_0+\eta_1),
 \end{align*}
  where we used $\sqrt{1-x} \le 1 - x/2$.
 Therefore,
 \[
 \max_{i \in \cI} \left| b_i - \bb_i \right|
 \wle 
 2^{3/2} \sqrt{ \frac{ n \bd }{ \eta_0+\eta_1 } } \cdot \frac{\alpha + \lambda}{\alpha} .
 \]
 Finally, Lemma~\ref{lemma:bound_ratio_sumbi_sumbisquared} ensures that
 \begin{align*}
   \frac{\sum_i \left| \bb_i \right| }{ \sum_i \bb_i^2 } 
   \wle \frac{2}{\sqrt{n} (\eta_1-\eta_0)} \cdot \frac{\lambda + \alpha}{\alpha \lambda} \left( 1 + \frac{2 \eta_0 n \sqrt{\eta_1+\eta_0}}{ \lambda } \right).
 \end{align*}
Therefore,
 \begin{align*}
   \frac{\sum_i  \left| \bb_i \right| \cdot \left| b_i - \bb_i \right| }{\sum_i \bb_i^2} 
   & \wle 2^{5/2} \left( \frac{\alpha+\lambda}{\alpha} \right)^2  \frac{ \sqrt{\bd} }{ (\eta_1-\eta_0)\sqrt{\eta_0+\eta_1} } \left( 1 + \frac{2 \eta_0 n \sqrt{\eta_1+\eta_0}}{ \lambda } \right) \\
   & \wle  2^{3} \left( \frac{\alpha+\lambda}{\alpha} \right)^2  \frac{ \sqrt{\bd} }{ (\eta_1-\eta_0)\sqrt{\eta_0+\eta_1} }. 
 \end{align*}
 where we used the condition $2 \eta_0 n \sqrt{\eta_1+\eta_0} \ll \lambda$.
 Going back to inequality~\eqref{eq:in_proof_bounding_delta_gamma}, this implies that
  $
  \left| \gamma_* - \bar{\gamma}_* \right|  \wle \left( 1 + \frac{27}{2^{6} } \frac{(\alpha+\lambda)^3}{\alpha^2 \lambda} \frac{1}{(\eta_1-\eta_0) \sqrt{\eta_0+\eta_1}} \right) \sqrt{ \bar{d} },
 $ and this concludes the proof. 
\end{proof}

\begin{lemma}
\label{lemma:estimations_di}
Let $-\E A_\tau + \lambda \cP = \bar{Q} \bar{\Delta} \bar{Q}^T$, where $\bar{\Delta} = \diag\left( \bdelta_1, \dots , \bdelta_n\right)$ and $\bar{Q}^T \bar{Q} = I_n$. Denote $\bb = \lambda \bar{Q}^T s$. We have
$
\bb_1 \wge \sqrt{n} \frac{\lambda (\eta_1 - \eta_0)}{2} \frac{ \bar{ \alpha } }{ \lambda + \bar{ \alpha } }.
$
Moreover, $\bb_i = 0$ if $\bdelta_i = 0$ or if $\bdelta_i = - t_1^-$.
\end{lemma}
\begin{proof}
First, from Corollary~\ref{cor:spectrum_Ltilde},
$
\bar{ \delta }_1 = -t_2^+ = - \frac{1}{2}\Bigg(  \bar{ \alpha } - \lambda + \sqrt{ \Big( \lambda + \bar{ \alpha } \Big)^2 - 4 \bar{ \alpha } \lambda \left( \eta_1 + \eta_0 \right) } \Bigg).
$
By symmetry, the $i$-th component of the first eigenvector $\bar{Q}_{\cdot 1}$ (associated with $\bar{\delta}_1$) is equal to
\begin{align*}
    \begin{cases}
        v_1 \, Z_i & \quad \text{if } i \in [\ell], \\
        v_0 \, Z_i   &\quad \text{if } i \not \in [\ell],
\end{cases}
\end{align*}
where $v_1$ and $v_0$ are to be determined.
Thus, the equation $\left( -\E A_\tau + \lambda \cP \right) \bar{Q}_{\cdot 1} = \bar{\delta}_1 \bar{Q}_{\cdot 1}$ leads to
\begin{align*}
    \begin{cases}
        \bar{ \alpha } \left( (\eta_1 + \eta_0) v_1 + (1 - \eta_1 - \eta_0) v_0 \right) & \weq -t_2^+ v_0 \\
        \bar{ \alpha } \left( (\eta_1 + \eta_0) v_1 + (1-\eta_1 - \eta_0) v_0 \right) + \lambda v_1   & \weq -t_2^+ v_1,
\end{cases}
\end{align*}
which, given the norm constraint $\|v\|_2 = 1$, yields
\begin{align*}
 \begin{cases}
  v_1 & \weq \dfrac{1}{\sqrt{n}}   \frac{ t_2^+ }{ \sqrt{ (\eta_1+\eta_0) \left( t_2^+ \right)^2 + (1-\eta_1 - \eta_0) \left( t_2^+ + \lambda \right)^2 } }, \\
  v_0 & \weq \dfrac{1}{\sqrt{n}} \frac{ + t_2^+ + \lambda}{ \sqrt{ (\eta_1+\eta_0) \left( t_2^+ \right)^2 + (1-\eta_1 - \eta_0) \left( t_2^+ + \lambda \right)^2 }}.
 \end{cases}
\end{align*}
Since $\bb_1 = \lambda v^T \bs = \lambda (\eta_1 - \eta_0) n v_1$, we have
\begin{align*}
 \frac{ \bb_1 }{ \sqrt{n} }
  & \weq \lambda (\eta_1 - \eta_0) \frac{t_2^+}{ \sqrt{ (\eta_1+\eta_0) \left( t_2^+ \right)^2 + (1-\eta_1 - \eta_0) \left( t_2^+ + \lambda \right)^2 } }.
\end{align*}
The proof ends by noticing that $t_2^+ \geq \frac{\bar{ \alpha }}{2}$ and $t_2^+ \leq \bar{ \alpha }$. Indeed,
\begin{align*}
 \frac{ \bb_1 }{ \sqrt{n} } & \wge \lambda(\eta_1 - \eta_0) \frac{ \bar{ \alpha } }{2 \sqrt{ (\eta_1 + \eta_0) \bar{ \alpha }^2 + (1-\eta_1-\eta_0) (\bar{ \alpha } + \lambda)^2 } } \\
 & \wge \frac{\lambda(\eta_1 - \eta_0)}{2} \frac{ \bar{ \alpha } }{ \left( \bar{ \alpha } + \lambda \right) \sqrt{ (\eta_1 + \eta_0) \left( \frac{\bar{ \alpha }}{\bar{ \alpha } + \lambda } \right)^2 + 1-\eta_1 - \eta_0 }  } \\
 & \wge \frac{\lambda (\eta_1 - \eta_0) }{ 2 } \frac{ \bar{ \alpha } }{ \lambda + \bar{ \alpha } }.
\end{align*}
This proves the first claim of the lemma.

Similarly, by symmetry the $i$-th component of the eigenvector $v'$ associated with $-t_1^-$ equals
$v_{\ell}'$ if $i \in \ell$, and $v_u'$ otherwise, and therefore $(v')^T s = 0$. 

Finally, let $I_0 := \{ i \in [n] : \bar{\delta}_i = 0 \}$. By Corollary~\ref{cor:spectrum_Ltilde}, we have $|I_0| = n (1-\eta_1-\eta_0) - 2$. Since $0$ is also eigenvalue of order $n(1-\eta_0-\eta_1)-2$ of the extracted sub-matrix $\left( -\E A_\tau + \lambda \cP \right)_{u,u} = \left( -\E A_\tau \right)_{u,u} $, we have for all $k \in I_0$, $\bar{Q}_{ik} = 0$ for every $i \in [n]$. Therefore, for $k \in I_0$, $b_k = \lambda \bar{Q}^T_{\cdot k} s = 0$.
\end{proof}

\begin{lemma}
\label{lemma:bound_ratio_sumbi_sumbisquared}
 Let $-\E A_\tau + \lambda \cP = \bar{Q} \bar{\Delta} \bar{Q}^T$, where $\bar{\Delta} = \diag\left( \bdelta_1, \dots , \bdelta_n\right)$ and $\bar{Q}^T \bar{Q} = I_n$. Denote $\bb = \lambda \bar{Q}^T s$ and let $\cI = \{ i \in [n] \colon \bb_i \not=0 \}$. We have
 $
 \frac{ \sum_{i \in \cI} |\bb_i| }{ \sum_{i \in \cI} |\bb_i|^2} \wle \frac{2}{\sqrt{n} (\eta_1-\eta_0)} \cdot \frac{\lambda + \alpha}{\alpha \lambda} \left( 1 + \frac{2 \eta_0 n \sqrt{\eta_1+\eta_0}}{ \lambda } \right)
 $. 
\end{lemma}
\begin{proof}
 Using Lemma~\ref{lemma:estimations_di}, we see that $\cI = \{ i \in [n] : \delta_i \not\in \{0, t_1^-\} \}$. Thus,
 \begin{align*}
     \frac{ \sum_{i \in \cI} |\bb_i| }{ \sum_{i \in \cI} |\bb_i|^2} 
     \weq \frac{|b_1| + \sum_{i \colon \delta_i = \lambda} |\bb_i| }{ |\bb_1|^2 + \sum_{i \colon \delta_i = \lambda} |\bb_i|^2}
     \wle \frac{ 1 }{ |\bb_1| } + \frac{ \sum_{i \colon \delta_i = \lambda} |\bb_i| }{ |\bb_1|^2}, 
 \end{align*}
 where $\bb_1$ denotes the element of vector $\bb$ corresponding to eigenvalue $\delta_1 = -t_2^+$. By Lemma~\ref{lemma:estimations_di}, we have $\bb_1 \wge \sqrt{n} \frac{\lambda (\eta_1 - \eta_0)}{2} \frac{ \bar{ \alpha } }{ \lambda + \bar{ \alpha } }$. Hence,
 \begin{align}
 \label{eq:first_inequation}
     \frac{ 1 }{ |\bb_1|}  
     \wle \frac{\lambda + \balpha}{ \balpha \lambda } \frac{2}{(\eta_1-\eta_0) \sqrt{n}}.
 \end{align}
 We note that the eigenvalue $\delta_i = \lambda$ is of multiplicity $\eta n - 2$. Let us denote by $\{v_i\}$ the corresponding $\eta n - 2$ orthonormal eigenvectors associated with eigenvalue $\lambda$. Let $v_{ij}$ denote the $j$-th entry of $v_i$. We notice from the block structure of $-\E A_{\tau} + \lambda \cP$ that $v_{ij} = 0$ if $j \notin \ell$. Moreover, if we let $\tv_i$ be the restriction of $v_i$ to $\ell$, then $\tv_i$ belongs to the kernel of $ \left( -\E A_\tau \right)_{\ell \ell}$. Therefore, $\sum_{j \in \ell} \tv_{ij} = 0$, and   
 \begin{align*}
     \bb_i \weq \lambda v_i^T s \weq -2 \lambda \sum_{j \in \ell_0} \tv_{ij}, 
 \end{align*}
 where $\ell_0 = \{ j \in \ell \colon s_i \ne z_i\}$ is the set of nodes mislabeled by the oracle. 
 Hence, 
 \begin{align*}
    \sum_{i \colon \delta_i = \lambda} |\bb_i| 
    & \weq 2 \lambda \sum_{ i \colon \delta_i = \lambda } \left| \sum_{j \in \ell_0} \tv_{ij} \right|  
    \wle 2 \lambda \sum_{i \in \ell} \sum_{j \in \ell_0} |\tv_{ij}| 
    \wle 2 \lambda \sum_{j \in \ell_0} \sqrt{|\ell|} \sqrt{\sum_{i\in \ell} |\tv_{ij}|^2} 
    \wle 2 \lambda |\ell_0| \sqrt{ |\ell|},
 \end{align*}
 where the last inequality follows from the fact that the matrix $(\tv_{ij})_{i,j}$ is orthogonal. 
 Hence, using $\bb_1 \wge \sqrt{n} \frac{\lambda (\eta_1 - \eta_0)}{2} \frac{ \bar{ \alpha } }{ \lambda + \bar{ \alpha } }$, $|\ell_0| = \eta_0$, and $|\ell| = (\eta_0+\eta_1) n$, we obtain  
 \begin{align*}
   \frac{ \sum_{i \colon \delta_i = \lambda} |\bb_i| }{ |\bb_1|^2 } \wle 4 \eta_0 \sqrt{n} \frac{\sqrt{\eta_1+\eta_0}}{ \eta_1-\eta_0} \frac{\lambda+\alpha}{\alpha}. 
 \end{align*}
 Combining the latter inequality with~\eqref{eq:first_inequation} leads to the desired result. 
\end{proof}

\section{Mean-field solution}
\label{appendix:solution_mean_field}

In this section, we calculate the solution $\bar{x}$ to the mean-field model and deduce from it the conditions to recover the clusters.

 \begin{proposition}
 \label{prop:noisy_data_MF_solution} 
  Suppose that $\tau > \pout$. Then the solution of Equation~\eqref{eq:SSL_solution_minimization_noisy} on the mean-field DC-SBM is the vector $\bar{x}$ whose element $\bar{x}_i$ is given by
  \begin{align*} 
   \bar{x}_i \weq 
   \begin{cases}
   C \left( - 1 + (\eta_1 - \eta_0) \bar{\alpha} B \right) Z_i, & \quad \text{if } i \in \ell \text{ and } s_i \not = Z_i, \\
   C \left( 1 + (\eta_1 - \eta_0) \bar{\alpha} B \right) Z_i, & \quad \text{if } i \in \ell \text{ and } s_i = Z_i, \\
   \frac{- \bar{\alpha} C}{ \bar{ \alpha } (1-\eta_1 - \eta_0) + \bar{\gamma}_* }  (\eta_1 - \eta_0) \left( 1 + (\eta_1 + \eta_0) \bar{\alpha} B \right) Z_i ,  & \quad \text{if } i \not \in \ell,	\end{cases}
  \end{align*}
  where $\bar{\alpha} = \frac{n}{2}(\pin - \pout)$,
  $B = \frac{ \bar{\alpha} \bgamma_* }{ \lambda \bar{\alpha} (1-\eta_1-\eta_0) + \bgamma_*(\lambda - \bar{ \alpha } - \bgamma_*)}$ and $C = \frac{\lambda}{\lambda - \bgamma_*}$.
 \end{proposition}
	
\begin{proof}
Let $\bar{x}$ be a solution of Equation~\eqref{eq:SSL_solution_minimization_noisy}. By symmetry, we have
\begin{align*}
\bar{x}_{i} \weq
\begin{cases}
x_t \, Z_i, & \quad \text{if } i \in [\ell] \text{ and } \bs_i = Z_i, \\
x_f \, Z_i, &\quad \text{if } i \in [\ell] \text{ and } \bs_i = - Z_i, \\
x_0 \, Z_i,   &\quad \text{if } i \not \in [\ell],
\end{cases}
\end{align*}
where $x_t$, $x_f$ and $x_0$ are unknowns to be determined.
Since for every $i \in [n]$
\begin{align*}
    \left( \E A_\tau \bar{x} \right)_i \weq \bar{ \alpha } \left( x_0 (1-\eta_1 - \eta_0) + x_t \eta_1 + x_f \eta_0 \right),
\end{align*}
the linear system composed of the equations
$
\big( \left( -\E A_\tau + \lambda \cP - \bgamma_* I_n \right) \bx \big)_i \weq \lambda s_i 
$ 
for all $i \in [n]$ leads to the system
\begin{align*}
\begin{cases}
- \balpha \left(  (1-\eta_1 - \eta_0) x_0 + x_t \eta_1 + x_f \eta_0 \right) - \bgamma_* x_0 & \weq 0, \\
- \balpha \left(  (1-\eta_1 - \eta_0) x_0 + x_t \eta_1 + x_f \eta_0 \right) - \bgamma_* x_t + \lambda x_t & \weq \lambda, \\
- \balpha \left(  (1-\eta_1 - \eta_0) x_0 + x_t \eta_1 + x_f \eta_0 \right) - \bgamma_* x_f + \lambda x_f & \weq - \lambda.
\end{cases}
\end{align*}
The rows of the latter system correspond to a node unlabeled by the oracle, correctly labeled and falsely labeled, respectively. 
This system can be rewritten as follows:
\begin{align*}
\begin{cases}
x_0 & \weq \frac{ - \bar{ \alpha } }{ \bar{ \alpha } (1-\eta_1-\eta_0) + \bgamma_* }  \left( \eta_1 x_t + \eta_0 x_f \right), \\
\bgamma_* x_0 + x_t (\lambda - \bgamma_*) & \weq \lambda, \\
\bgamma_* x_0 + x_f (\lambda - \bgamma_*) & \weq - \lambda.
\end{cases}
\end{align*}
In particular, we have
$
 x_t - x_f \weq \frac{ 2\lambda }{ \lambda - \bgamma_* }.
$
By subsequently eliminating $x_0$ and $x_t$ in the equation $\bgamma_* x_0 + x_f (\lambda - \bgamma_*) = - \lambda$, 
we find 
\begin{align*}
 x_f & \weq \frac{\lambda}{\lambda - \bgamma_*} \left( -1 + \frac{ \bar{ \alpha } \bgamma_* \left( \eta_1 - \eta_0 \right) }{ \lambda \bar{ \alpha } (1-\eta_1 - \eta_0) + \lambda \bgamma_* - \bgamma_* ( \bar{ \alpha } + \bgamma_* ) } \right), \\
 x_t & \weq \frac{ \lambda }{ \lambda - \bgamma_* } \left( 1 + \frac{ \bar{ \alpha } \bgamma_* \left( \eta_1 - \eta_0 \right)}{ \lambda \bar{ \alpha } (1-\eta_1 - \eta_0) + \lambda \bgamma_* - \bgamma_* ( \bar{ \alpha } + \bgamma_* ) }\right),
\end{align*}
and finally
\begin{align*}
 x_0 \weq \frac{ - \balpha }{ \bar{ \alpha }(1-\eta_1-\eta_0) + \bgamma_* } \cdot \frac{ \lambda }{ \lambda - \bgamma_*} \left( 1 + \frac{ \balpha \bgamma_* \left( \eta_1 + \eta_0 \right)}{ \lambda \balpha (1-\eta_1 - \eta_0) + \lambda \bgamma_* - \bgamma_* ( \balpha + \bgamma_* ) }\right).
\end{align*}
\end{proof}

\begin{corollary}
\label{cor:recovery_with_the_mean_field}
 Suppose that $\tau > \pout $.
 Then $\mathrm{sign}\left( \bx_i \right) = \mathrm{sign}\left( Z_i \right)$ if
 \begin{itemize}
  \item node $i$ is not labeled by the oracle;
  \item node $i$ is correctly labeled by the oracle;
  \item node $i$ is mislabeled by the oracle and $\lambda < (1-2\eta_0) \bar{\alpha} \frac{\eta_1-\eta_0}{\eta_1+\eta_0}$.
 \end{itemize}
\end{corollary}
	
\begin{proof}
 A node $i$ is correctly classified by decision rule~\eqref{detection_criteria} if the sign of~$\bx_i$ is equal to the sign of~$Z_i$.
 Using Lemma~\ref{lemma:bounding_gamma_star_mf} in Appendix~\ref{appendix:estimation_gamma_starMF}, we have $- \balpha \le \bgamma_* \le -\balpha (1 - 2 \eta_0)$. 
 Therefore, the quantities $B$ and $C$ in Proposition~\ref{prop:noisy_data_MF_solution} verify $C \ge 0$ and $\frac{1-2\eta_0}{\lambda(\eta_0+\eta_1)}
  \le B 
  \le \frac{1}{\lambda(\eta_0+\eta_1)}$.
 The statement then follows from the expression of $\bar{x}_i$ computed in Proposition~\ref{prop:noisy_data_MF_solution}.
\end{proof}

\section{Cost comparison in the constrained eigenvalue problem}

\begin{lemma}
 \label{lemma:diff_cost}
 Let $(\gamma_1,x_1)$ and $(\gamma_2,x_2)$ be two solutions of the system~\eqref{eq:constraint_linear_system}, and denote by $\cC(x) = - x^T A_\tau x + \lambda (s-\cP x)^T (s-\cP x)$ the cost function minimized in~\eqref{eq:MAP_relaxed}. Then, we have 
\begin{align*}
 \cC (x_1) - \cC (x_2) \weq \frac{ \gamma_1 - \gamma_2 }{2} \, \left\| x_1 - x_2 \right\|^2.
\end{align*}
\end{lemma}

\begin{proof}
Because $(\gamma_1,x_1)$ and $(\gamma_2,x_2)$ are solutions of \eqref{eq:constraint_linear_system}, it holds that
\begin{align}
\label{eq:in_proof_cost1}
 \left( - A_\tau + \lambda \cP \right)  x_1 \weq & \gamma_1 x_1 + \lambda s, \\
 \label{eq:in_proof_cost2}
  \left( - A_\tau + \lambda \cP \right)  x_2 \weq & \gamma_2 x_2 + \lambda s,
\end{align}
as well as $\|x_1\|^2 = \|x_2\|^2=n$.  
Thus, we notice that
\begin{align*}
 \cC(x_1) & \weq x_1^T \left( - A_{\tau} + \lambda \cP \right) x_1 + \lambda s^T s - 2 \lambda \, x_1^T \cP s \\
 & \weq - \lambda x_1^T s + \gamma_1 n + \lambda s^T s,
\end{align*}
where we used $\cP s = s$ and the fact that $(\gamma_1, x_1 )$ is a solution of the system~\eqref{eq:constraint_linear_system}. Therefore, 
\begin{align*}
 \cC(x_1) - \cC(x_2) \weq \left( \gamma_1 - \gamma_2 \right) n + \lambda \left( x_2 - x_1 \right)^T s.
\end{align*}
Finally, by multiplying on the left Equation~\eqref{eq:in_proof_cost1} by $x_2^T$ (resp., Equation~\eqref{eq:in_proof_cost2} by $x_1^T$), we obtain
\begin{align*}
 \lambda x_2^T s & \weq x_2^T \left( - A_\tau + \lambda \cP \right)  x_1 - \gamma_1 x_2^T x_1 , \\
 \lambda x_1^T s & \weq x_1^T \left( - A_\tau + \lambda \cP \right)  x_2 - \gamma_2 x_1^T x_2.\\
\end{align*}
Thus,
\begin{align*}
  \cC(x_1) - \cC(x_2) 
  & \weq \left( \gamma_1 - \gamma_2 \right) \left( n - x_1^T x_2 \right) \\
  & \weq \frac{ \gamma_1 - \gamma_2 }{ 2 } \left( \|x_1\|^2 + \|x_2\|^2 - 2 x_1^T x_2 \right) \\
  & \weq \frac{ \gamma_1 - \gamma_2 }{2} \, \left\| x_1 - x_2 \right\|^2,
\end{align*}
where we used the constraints $\|x_1\|^2 = \|x_2\|^2=n$.
\end{proof}
\color{black}

\paragraph{Funding statement}
This research has been done within the project of Inria - Nokia Bell Labs ``Distributed Learning and
Control for Network Analysis''.

\paragraph{Data availability statement}
Data can be found for instance at Stanford Large Network Dataset Collection \url{https://snap.stanford.edu/data/} and the code is available on GitHub \url{https://github.com/mdreveton/ssl-sbm}.

\end{document}